\documentclass[twoside,11pt]{article}

\usepackage[preprint]{jmlr2e}

\jmlrheading{}{}{}{}{}{}{Yuwei Luo, Zhuoran Yang, Zhaoran Wang, and Mladen Kolar}

\ShortHeadings{Natural Actor-Critic for Hierarchical Linear Quadratic Regulator}{Luo, Yang, Wang and Kolar}
\firstpageno{1}

\usepackage{algorithm}
\usepackage{algorithmic}
\usepackage{amsmath}
\usepackage{dsfont}
\usepackage{mathtools}
\usepackage{enumerate}

\def\vec{\mathop{\text{vec}}}
\def\svec{\mathop{\text{svec}}}

\def\smat{\mathop{\text{smat}}}

\newcommand{\minimize}{\mathop{\mathrm{minimize}}}

\newcommand{\tr}{\mathop{\mathrm{tr}}}

\def\real{{\mathbb{R}}}
\def\R{{\real}}

\ifx\assumption\undefined
\newtheorem{assumption}{Assumption}
\fi

\newtheorem{theorem}{Theorem}
\newtheorem{lemma}{Lemma} 
\newtheorem{proposition}{Proposition}

\newtheorem{definition}{Definition}

\def\##1\#{\begin{align}#1\end{align}}
\def\$#1\${\begin{align*}#1\end{align*}}
\DeclareMathOperator{\ind}{\mathds{1}}  
\newcommand{\diag}{{\rm diag}}

\usepackage{caption}
\usepackage{tabularx} 
\usepackage{enumitem} 
\usepackage{xcolor}

\usepackage{xargs}
\usepackage[colorinlistoftodos,prependcaption,textsize=tiny]{todonotes}
\newcommandx{\unsure}[2][1=]{\todo[linecolor=red,backgroundcolor=red!25,bordercolor=red,#1]{#2}}
\newcommandx{\change}[2][1=]{\todo[linecolor=blue,backgroundcolor=blue!25,bordercolor=blue,#1]{#2}}
\newcommandx{\info}[2][1=]{\todo[linecolor=OliveGreen,backgroundcolor=OliveGreen!25,bordercolor=OliveGreen,#1]{#2}}
\newcommandx{\improvement}[2][1=]{\todo[linecolor=Plum,backgroundcolor=Plum!25,bordercolor=Plum,#1]{#2}}

\mathtoolsset{showonlyrefs} 
\allowdisplaybreaks

\begin{document}

\title{Natural Actor-Critic Converges Globally for\\ Hierarchical Linear Quadratic Regulator}

\author{\name Yuwei Luo \email yuweiluo@stanford.edu \\
       \addr Graduate School of Business\\
        Stanford University\\
	   Stanford, CA 94305, USA
       \AND
       \name Zhuoran Yang \email zy6@princeton.edu \\
       \addr Department of Operations Research and Financial Engineering\\
       Princeton University\\
       Princeton, NJ 08544, USA
       \AND
       \name Zhaoran Wang \email zhaoran.wang@northwestern.edu \\
		\addr Department of Industrial Engineering and Management Sciences\\
		Northwestern University\\
		Evanston, IL 60208, USA
       \AND		
		\name Mladen Kolar \email mkolar@chicagobooth.edu \\
		\addr The University of Chicago Booth School of Business\\
		Chicago, IL 60637, USA  
}

\editor{}

\maketitle

\begin{abstract}
  Multi-agent reinforcement learning has been successfully applied to a number of challenging problems. Despite these empirical successes, theoretical understanding of different algorithms is lacking, primarily due to the curse of dimensionality caused by the exponential growth of the state-action space with the number of agents. We study a fundamental problem of multi-agent linear quadratic regulator (LQR) in a setting where the agents are partially exchangeable. In this setting, we develop a hierarchical actor-critic algorithm, whose computational complexity is independent of the total number of agents, and prove its global linear convergence to the optimal policy. As LQRs are often used to approximate general dynamic systems, this paper provides an important step towards a better understanding of general hierarchical mean-field multi-agent reinforcement learning.
\end{abstract}

\begin{keywords}
  Markov Decision Process, Linear Quadratic Regulator, Actor-Critic Algorithms, Multi-Agent Reinforcement Learning, Mean-Field Reinforcement Learning
\end{keywords}

\section{Introduction}

Multi-agent reinforcement learning (MARL) \citep{bu2008comprehensive} combined with deep neural networks has recently been applied successfully to problems ranging from self-driving cars \citep{shalev2016safe} and robotics \citep{yang2004multiagent} to E-sports \citep{alphastarblog,OpenAI_dota} and Go \citep{silver2016mastering,silver2017mastering}.  Despite promising empirical results in few specific domains, MARL remains challenging both in theory and practice as the state-action space grows exponentially with the number of agents \citep{menda2019deep}.  This curse of dimensionality makes developing computationally tractable and statistically consistent procedures difficult.  When the agents are homogeneous, the curse of dimensionality can be avoided by exploiting symmetries in the problem, which gives rise to mean-field multi-agent reinforcement learning \citep{huang2012social,carmona2013control,fornasier2014mean}. Mean-field algorithms rely on the assumption that the agents are exchangeable, in which case the optimal policy that maximizes the expected total reward symmetrically decomposes across agents.  As a consequence, the optimal policy can be found by solving a single-agent reinforcement learning problem while additionally accounting for the mass effect induced by all other agents, which is summarized by a mean field. Through this reduction to a single-agent reinforcement learning problem one can again obtain computationally tractable procedures for which the statistical error does not grow exponentially with the number of agents \citep{yang2018mean}.

The assumption that the agents are exchangeable is often violated in practical problems, such as real-time strategy gaming with different kinds of units \citep{OpenAI_dota, alphastarblog} and urban traffic control (UTC) with heterogeneous junctions \citep{el2012multi, chu2017traffic}, which makes practical application of the mean-field algorithms difficult. One approach to relaxing the exchangeability assumption is through the notion of \emph{partial exchangeability} \citep{arabneydi2016linear}, which allows for the exploitation of the symmetry among possibly heterogeneous agents. The key to partial exchangeability is the hierarchical structure of agents, which is often observed in practice.  Within a subpopulation of exchangeable agents, the symmetry is exploited as in mean-field multi-agent reinforcement learning through decoupling, while the heterogeneity across different subpopulations of agents is accounted for by tracking multiple mean fields. In particular, within each subpopulation of agents, it suffices to solve a single-agent reinforcement learning problem. Due to partial exchangeability, one can escape the curse of dimensionality, while allowing for the heterogeneity among agents.

Our contribution is three-fold. First, we motivate the hierarchical LQR model via decoupling the dynamics of agents in the system with the notion of partial exchangeability, thus decomposing the multi-agent LQR control problem into computationally tractable control problems on subpopulation systems. Second, we extend MARL approaches for the hierarchical LQR model by proposing a hierarchical actor-critic algorithm that is model-free, with computational complexity independent of the number of agents in each subpopulation, thus breaking the curse of dimensionality. Third, we establish non-asymptotic global rate of convergence of our algorithm for the multi-agent LQR control problem, which is fundamental in MARL and optimal control. 

\subsection{Related Work}
We contribute to several strands of the literature, including the development of actor-critic algorithms, LQRs,  distributed control, and mean-field multi-agent reinforcement learning.  

Our algorithm belongs to the family of actor-critic algorithms. \citet{konda1999actor} proposed the first actor-critic algorithm, which was later extended to the natural actor-critic algorithm \citep{peters2008natural} using the natural policy gradient \citep{kakade2002natural}.  Convergence analysis of actor-critic and natural actor-critic algorithms with linear function approximation was studied in \citet{kakade2002natural}, \citet{bhatnagar2009natural}, \citet{bhatnagar2008incremental}, \citet{castro2010convergent}, and \citet{bhatnagar2010actor}. Compared to the policy gradient algorithm \citep{williams1992simple}, the online (critic) update of the action-value function in an actor-critic algorithm reduces the variance of the policy gradient and leads to faster convergence, which was rigorously shown for LQR problem \citep{yang2019global}. Due to its favorable properties, in this paper, we develop a hierarchical natural actor-critic algorithm for the multi-agent LQR setting and establish linear global convergence to the optimal policy.

We establish our theoretical results in the setting of LQR, which is a fundamental problem in reinforcement learning and optimal control problems.   In LQR problems, the dynamics is approximated by a linear function and the cost is approximated by a quadratic loss.  LQR serves as a powerful model for optimal control problems and has achieved tremendous success in real-world problems such as Unmanned Aerial Vehicle (UAV) \citep{zhi2017flight,setyawan2019linear} and Power Grids \citep{minciardi2011optimal, vinifa2016linear}. In terms of theory, the optimal policy in the LQR setting takes a linear form \citep{zhou1996robust, 	anderson2007optimal, bertsekas2012dynamic} and a number of properties of reinforcement learning algorithms were established in this setting \citep{bradtke1993reinforcement, recht2018tour, tu2017least, tu2018gap, dean2018regret, dean2018safely, simchowitz2018learning, dean2017sample, hardt2018gradient}.  See \citet{recht2019tour} for a recent review. We contribute to this literature by studying the multi-agent LQR problem with partial exchangeability. In particular, our convergence analysis is inspired by the optimization landscape of LQR characterized by \citet{fazel2018global}, where they show the global convergence of policy gradient algorithm, and the global convergence analysis of the natural actor-critic algorithm established in the single agent LQR problem \citep{yang2019global}. Instead, we establish the global convergence of actor-critic algorithm for the multi-agent LQR control problem, while at the same time still being computationally tractable.

We further contribute to the literature on MARL in the framework of Markov games \citep{littman1994markov}.  A number of authors have tried to address the curse of dimensionality in MARL. \citet{wang2003reinforcement} and \citet{arslan2016decentralized} assumed that the rewards are identical among agents, and, as a result, no interaction needs to be considered. Linear function approximation methods were studied in \citet{lee2018primal} and \citet{zhang2018fully}, while function approximation with deep neural networks was explored in \citet{foerster2016learning}, \citet{gupta2017cooperative}, \citet{lowe2017multi}, \citet{omidshafiei2017deep}, and \citet{foerster2017stabilising}. These papers primarily focus on the empirical performance of the algorithms or establish asymptotic results, leaving the theoretical understanding and rigorous convergence analysis for MARL largely open.

Model-based approaches to the mean-field approximation require knowledge of model parameters \citep[see, for example,][]{elliott2013discrete, arabneydi2015team, 	arabneydi2016linear, li2017class}, while model-free methods \citep{yang2017learning,yang2018mean} only come with algorithms with asymptotic analysis.  In contrast, our method is model-free and comes with provable global non-asymptotic convergence analysis.  Mean-field MARL problem in a collaborative setting, known as team games \citep{tan1993multi, panait2005cooperative, wang2003reinforcement, claus1998dynamics}, can be regarded as a centralized mean-field control problem \citep{huang2012social, carmona2013control, fornasier2014mean} with infinitely many homogeneous agents. Our work extends this model by allowing potential heterogeneity among agents.

Our method is also related to distributed control problems.  To escape the curse of dimensionality in large-scale systems, a sequence of papers on distributed control assume homogeneous agents. Specifically,  \citet{borrelli2008distributed} studied identical and decoupled dynamics, \citet{massioni2009distributed, deshpande2012sub, alemzadeh2019distributed} focused on identical agents with coupling due to interconnection of the subsystems either through dynamics or common goals. These works differ from our setting in that they do not allow heterogeneity among agents, while our partially heterogeneous system includes a pure homogeneous system as a special case. \citet{vlahakis2018distributed} and \citet{sturz2020distributed} recently studied multi-agent systems with heterogeneous subsystems under some structural assumptions. However, their methods are model-based and, hence, are less likely to be extended to more general settings. In the context of distributed control, our work contributes to the development of model-free methods for multi-agent systems with heterogeneity. 

Our method relies on the notion of partial exchangeability \citep{arabneydi2016linear}, where they studied multi-agent LQR problems with partial exchangeability to exploit the symmetries in the problem.  A different notion with the same name was used to construct the joint state-action statistic, which can be combined with individual state-action pairs to predict the agent's next state \citep{nguyen2018credit}.  As such, it can be viewed as a generalization of the homogeneity of all agents. In contrast, partial exchangeability in our work assumes homogeneity within each subpopulation, thus allowing for heterogeneity across different subpopulations of agents. 
\citet{rahmani2009controllability}  explored how homogeneity and symmetry are related to the controllability of multi-agent systems by introducing network equitable partitions. They focus on Laplacian-based dynamics on the graph, which results in a special case of linear dynamical systems.

\subsection{Notation}\label{sec:notation}
For a vector $v$, we use $\|v\|_2$ to denote its $\ell_2$-norm. For a matrix $A$, we denote by $\|A\|$ and $\|A\|_{\mathrm{F}}$ its operator norm and Frobenius norm respectively. For a square matrix $X$, we use $\sigma_{\min}(X)$ and $\rho(X)$ to denote its minimal singular value and spectral radius respectively. For vectors $x$, $y$ and $z$, we denote by $\vec(x,y,z)$ the vector obtained by stacking all the vectors, i.e. $\left[x^{\top}, y^{\top}, z^{\top}\right]^{\top}$. We adopt the notation $A\succ 0$ ($A  \succeq 0$) for symmetric positive definite (positive semi-definite) matrix $A$, and $A\prec 0$ ($A  \preceq 0$) for symmetric negative definite (negative semi-definite) matrix $A$. We denote by $\operatorname{rows}(A, B, C)$ the matrix $\left[A^{\top}, B^{\top}, C^{\top}\right]^{\top}$ for matrices $A$, $B$ and $C$ with the same number of columns. Also, we denote by $\operatorname{cols}(A, B, C)$ the matrix $\left[A, B, C\right]$ for matrices $A$, $B$ and $C$ with the same number of rows. For a symmetric matrix $Z$, we denote by $\svec(Z)$ the vectorization of the upper triangular submatrix of $Z$, with the off-diagonal entries weighted by $\sqrt{2}$. The inverse operation is denoted by $\smat(\cdot)$.

\section{Background}
\label{section:background}

We provide the necessary background in this section. In \S\ref{sec:act_crit}, we describe actor-critic algorithm. LQR is introduced in \S\ref{sec:lqr}. \S\ref{sec:multi-agent} presents multi-agent reinforcement learning.

\subsection{Actor-Critic Algorithm}
\label{sec:act_crit}

In reinforcement learning, a system is described by a Markov decision process $\left\{\mathcal{X}, \mathcal{U}, c, \mathcal{T},\mathcal{D}_0 \right\}$, where starting with initial state $x_0\sim\mathcal{D}_0$, at each time step, an agent interacts with the environment by selecting an action $u_t \sim \pi(\cdot \mid x_t)\in \mathcal{U}$ based on its current state $x_t \in \mathcal{X}$.  Then the environment gives feedback with cost $c(x_t,u_t)$, and the agent moves to the next state by the transition kernel $x_{t+1} = \mathcal{T}(x_t,u_t)$. The agent aims to find the policy that minimizes the expected time-average cost 
\begin{equation}
\label{eq:expected_cost}
C(\pi) = \limsup _{T \rightarrow \infty} T^{-1} \sum_{t=0}^{T-1} \mathbb{E}_{x_0 \sim \mathcal{D}_{0}, u_t  \sim \pi(\cdot \mid x_t)}\left[c\left(x_{t}, u_{t}\right)\right].
\end{equation} 
Given any policy $\pi$, the action- and state-value functions are defined, respectively, as 
\begin{align*}
Q_{\pi}(x, u)&=\sum_{t } \mathbb{E}_{u_t  \sim \pi(\cdot \mid x_t)}\left[c\left(x_{t}, u_{t}\right) -C(\pi) \mid x_{0}=x, u_{0}=u\right], \\
V_{\pi}(x)&=\mathbb{E}_{u \sim \pi(\cdot|x)}\left[Q_{\pi}(x, u)\right].   
\end{align*}
In practice, the policy $\pi$ is parameterized as $\pi_\theta$. We denote the corresponding cost and action-value functions by $C(\theta)$ and $Q_\theta$, respectively.  By the policy gradient theorem \citep{sutton1999policy, baxter2001infinite},  for any MDP and any differentiable policy $\pi_{\theta}$, the gradient with respect to the parameter $\theta$ can be computed as
\[
\nabla_{\theta} C(\theta)=\mathbb{E}_{x \sim \mathcal{D}_{\theta},
	u \sim \pi_\theta(\cdot \mid x)}\left[\nabla_{\theta} \log \pi_\theta(u \mid x) \cdot Q_{\theta}(x, u)\right],
\]
where $\mathcal{D}_{\theta}$ is the stationary distribution of the Markov chain $\{x_t\}_{t\geq0}$ under policy $\pi_\theta$.

An actor-critic algorithm consists of a critic step that approximates the action-value function $Q_{\theta}$ with a parameterized function $Q(\cdot,\cdot,\omega)$ by estimating the parameter $\omega$, and an actor step where the policy $\pi_\theta$ is updated with a stochastic version of the policy gradient.  The natural actor-critic algorithm \citep{peters2008natural} updates the policy with the natural policy gradient $\left[\mathcal{I}_{\theta}\right]^{-1} \nabla_{\theta} C(\theta)$ \citep{kakade2002natural}, where
\[
\mathcal{I}_{\theta} = \mathbb{E}_{x \sim \mathcal{D}_{\theta}, u \sim
	\pi_\theta(\cdot \mid x)}\left[\nabla \log \pi_\theta(u \mid x) \cdot \nabla
\log \pi_\theta(u \mid x)^{\top}\right]
\]
is the Fisher information of the policy $\pi_\theta$.

\subsection{Linear Quadratic Regulator}
\label{sec:lqr}

LQR is a fundamental problem in optimal control. In reinforcement learning, the LQR setting is used to develop theoretical understanding of different methods and hence serves as a performance benchmark \citep{fazel2018global,tu2017least,hardt2018gradient}.  The state space is specified as $\mathcal{X} = \mathbb{R}^{d}$ and action space as $\mathcal{U} = \mathbb{R}^{k}$.  The transition dynamics takes a linear form and the cost function takes a quadratic form, specified by 
\begin{equation}
\label{eq:LQR} x_{t+1}=A x_{t} + B u_{t} + w_t, \quad
c\left(x_{t}, u_{t}\right)=x_{t}^{\top} Q x_{t}+u_{t}^{\top} R u_{t},
\end{equation}
where  $A\in \real^{ d\times d}$, $B \in \real^{ d\times k} $, $Q\in \real^{ d\times d}$ and $R\in \real^{k\times k}$ are matrices of appropriate dimensions, the noise $w_t \sim N(0,\Phi)$, and
$Q, R, \Phi \succ 0$.

The policy $\pi^*$ that minimizes $C(\pi)$ in \eqref{eq:expected_cost} is deterministic and static, taking the linear form \citep{anderson2007optimal} $\pi^*(x_t) = -K^* x_t$, where
\[
K^*  = \left(R+B^{\top} P^{*} B\right)^{-1} B^{\top} P^{*} A,
\]
and $P^*$ is the solution to a discrete algebraic Riccati equation \citep{zhou1996robust}.  In the model-free setting, where reinforcement learning methods do not have access to model parameters, it is known that policy gradient \citep{fazel2018global,tu2018gap,malik2018derivative} and actor-critic \citep{kakade2002natural} are guaranteed to find the optimal policy $\pi^*$.

\subsection{Multi-Agent Reinforcement Learning}
\label{sec:multi-agent}

A multi-agent system with the set of agents $\mathcal{N}$ can be described by a Markov Decision Process (MDP) characterized by the tuple $\left(\mathcal{X}, \{\mathcal{U}^i\}_{i\in \mathcal{N}}, \mathcal{T}, \{c_i\}_{i\in \mathcal{N}}\right)$ \citep{littman1994markov}, where each agent in the system takes an individual action $u_i \in \mathcal{U}^i$, observes its individual cost $c_i:\mathcal{X}\times \mathcal{U}\rightarrow \mathbb{R}$, and moves to the next state by the global transition kernel $\mathcal{T}:\mathcal{X}\times \mathcal{U}\rightarrow \mathcal{X}$, with $\mathcal{U} = \prod_i \mathcal{U}^i$ denoting the joint action space.

We are interested in the team optimal control problem where the goal is to find a parameterized policy $\pi=\pi_\theta:\mathcal{X}\times \mathcal{U} \rightarrow \mathcal{U}$ that minimizes the global total expected time-average cost of all agents, defined by
\begin{equation*}
\minimize_{\theta} \ C(\theta)=\limsup _{T \rightarrow \infty} T^{-1} \sum_{t=0}^{T-1} \mathbb{E}\left[ \sum_{i \in \mathcal{N}}c_i\left(\mathbf{x}_{t}, \mathbf{u}_{t}\right)\right]=\limsup _{T \rightarrow \infty} T^{-1} \sum_{t=0}^{T-1} \mathbb{E}\left[ c_{gt}\left(\mathbf{x}_{t}, \mathbf{u}_{t}\right)\right].
\end{equation*}
Here $c_{gt}\left(\mathbf{x}_{t}, \mathbf{u}_{t}\right) = \sum_{i \in 	\mathcal{N}}c_i\left(\mathbf{x}_{t}, \mathbf{u}_{t}\right)$ is the global total cost, while $\mathbf{x}_t = ({x_t^1, \dots ,x_t^{\left|\mathcal{N}\right|}})$ and $\mathbf{u}_t = ({u_t^1, \dots ,u_t^{\left|\mathcal{N}\right|}})$ denote the joint state and action tuples, respectively.  The corresponding action-value function is defined as
\begin{equation*}
Q_{\theta}(\mathbf{x}, \mathbf{u})=\sum_{t } \mathbb{E}_{ \theta}\left[c_{gt}\left(\mathbf{x}_{t}, \mathbf{u}_{t}\right) -C(\theta) \mid \mathbf{x}_{0}=\mathbf{x}, \mathbf{u}_{0}=\mathbf{u}\right]
\end{equation*}
and the state-value function is given by $V_{\theta}(\mathbf{x})=\mathbb{E}_{\mathbf{u} \sim \pi_\theta(\cdot \mid 	\mathbf{x})}\left[Q_{\theta}(\mathbf{x},\mathbf{u})\right]$.  The action- and state-value functions are coupled across agents since the transition dynamics and costs depend on the joint state and action of the entire system. As the number of agents increases, it becomes infeasible to learn $Q_{\theta}(\cdot, \cdot)$ due to the coupling structure and the exponentially increasing interactions.

\section{Hierarchical Mean-Field Multi-Agent Reinforcement Learning}
\label{section:hmrl}

We consider the setting where the population of agents $\mathcal{N}$ satisfies the property that it can be partitioned into disjoint subpopulations $\mathcal{N} = \bigsqcup_l^L \mathcal{N}^l$, such that agents within each subpopulation are \emph{exchangeable} (see Definition~\ref{def:exchangeability}). We introduce \emph{Hierarchical Mean-Field Multi-Agent Reinforcement Learning}, where we approximate the interactions of agents by each agent interacting with mean-field effects of subpopulations, as a way for accounting for the heterogeneity of agents and dealing with the curse of dimensionality. In \S\ref{sec:partial_exchangeability}, we define system with \emph{partial exchangeability}. \emph{Hierarchical Actor-Critic 	Algorithm} is introduced in~\S\ref{sec:hierarchical_ac}.

\subsection{Partial Exchangeability}
\label{sec:partial_exchangeability}

Consider a multi-agent dynamical system with agents $\mathcal{N}$. The state, action, and noise spaces for each agent are specified by $\mathcal{X}^i=\mathbb{R}^{d_i}$, $\mathcal{U}^i=\mathbb{R}^{k_i}$, and $\mathcal{W}^i=\mathbb{R}^{d_i}$ respectively.  Let $\mathbf{x}_{t}=\operatorname{vec}\left(\left({x}_{t}^{i}\right)_{i 	\in \mathcal{N}} \right)$, $\mathbf{u}_{t}=\operatorname{vec}\left(\left({u}_{t}^{i}\right)_{i 	\in \mathcal{N}} \right)$, and $\mathbf{w}_{t}=\operatorname{vec}\left(\left({w}_{t}^{i}\right)_{i	\in \mathcal{N}} \right)$ denote  the global state, action, and noise vectors of the whole system at time $t$ respectively.  The system transition dynamics is given by
\begin{equation}
\label{eq:general_dymcs}
\mathbf{x}_{t+1}=f\left(\mathbf{x}_{t}, \mathbf{u}_{t},
\mathbf{w}_{t}\right).
\end{equation}
Let $c_{t}\left(\mathbf{x}_{t}, \mathbf{u}_{t}\right)$ denote the per-step cost at time $t$ and $\sigma_{i,j}$ denote the permutation transformation. For example, $\sigma_{1,3}((x_1,x_2,x_3)) = (x_3,x_2,x_1)$. We first give the definition of \emph{partial exchangeability}
\citep{arabneydi2016linear}.
\begin{definition}[Exchangeable agents]
	\label{def:exchangeability}
	A pair $(i,j)$ of agents is called \emph{exchangeable} if 	$\mathcal{X}^{i}=\mathcal{X}^{j}$, 	$\mathcal{U}^{i}=\mathcal{U}^{j},$ and $\mathcal{W}^{i}=\mathcal{W}^{j}$. That is, the dimensions of state, action, and disturbance spaces are the same. Moreover, the dynamics and cost satisfy
	\begin{align*}
	f\left(\sigma_{i,j}\mathbf{x}_{t}, \sigma_{i,j}\mathbf{u}_{t}, \sigma_{i,j} \mathbf{w}_{t}\right) =
	&\sigma_{i,j}(f\left(\mathbf{x}_{t}, \mathbf{u}_{t}, \mathbf{w}_{t}\right)), \\
	c\left(\sigma_{i,j} \mathbf{x}_{t}, \sigma_{i,j} \mathbf{u}_{t}\right) =
	&c\left(\mathbf{x}_{t}, \mathbf{u}_{t}\right).
	\end{align*}
	That is, exchanging agents $i$ and $j$ does not affect the dynamics and cost.
\end{definition}
\begin{definition}[Multi-agent system with partial exchangeability]\label{partial_exchangeability}
A multi-agent system is called a system with \emph{partial 		exchangeability} if the agents $\mathcal{N}$ can be partitioned into $L$ disjoint exchangeable subpopulations $\mathcal{N}^{l}$, $l \in \mathcal{L} :=\{1, \ldots, L\}$, such that each pair of agents in $\mathcal{N}^{l}$ is exchangeable.
\end{definition}
Partial exchangeability assumes that agents in the system can be partitioned into subpopulations and exchanging agents in the same subpopulation does not affect the system dynamics and cost. This definition accounts for the heterogeneity among agents across subpopulations and, thus, applies to a broader range of settings compared to vanilla mean-field MARL methods, which assume homogeneous agents. With partial exchangeability, we define the mean-field of each subpopulation, which serves as a good summary of the information of that subpopulation.

\begin{definition}[Mean-fields of states and actions]\label{mean_field_states}
	The mean-field state and action of each subpopulation are defined, respectively, as the empirical means
	\[
	\bar{x}_{t}^{l} :=\frac{1}{\left|\mathcal{N}^{l}\right|} \sum_{i \in \mathcal{N}^{l}} x_{t}^{i}, \quad \bar{u}_{t}^{l} :=\frac{1}{\left|\mathcal{N}^{l}\right|} \sum_{i \in \mathcal{N}^{l}} u_{t}^{i}, \quad l \in \mathcal{L}.
	\]
	The global mean-field of the system is defined by stacking the mean-field value vectors:
	\begin{equation}
	\label{def:gl_mf}
	\bar{\mathbf{x}}_{t}:=\operatorname{vec}\left(\bar{x}_{t}^{1}, \ldots, \bar{x}_{t}^{L}\right),\quad \bar{\mathbf{u}}_{t}:=\operatorname{vec}\left(\bar{u}_{t}^{1}, \ldots, \bar{u}_{t}^{L}\right).
	\end{equation}  
\end{definition}

In \S\ref{section:main_result} we show that in the LQR setting, partial exchangeability makes the global mean-field values $\bar{\mathbf{x}}_{t}$ and $\bar{\mathbf{u}}_{t}$ sufficient to characterize the interactions of agents.

  We consider the information structure where each agent $i$ can perfectly observe its local state $x_t^i$, action ${u}_{t}^{i}$, and the global mean-field state $\bar{\mathbf{x}}_{t}$ defined in \eqref{def:gl_mf}. Furthermore, each agent can recall its entire observation history perfectly. Such an information structure is called \emph{mean-field sharing} \citep{arabneydi2016linear}. 

The assumption that each agent observes the global mean-field is called mean-field sharing. It is regarded as a \emph{non-classical} information structure in the related literature \citep{witsenhausen1971separation, ho1972team}. For LQR problems with \emph{classical} information structure, where each agent knows all observations and actions of other agents who act before it, and the \emph{partially nested} information structure, where the agents observe all observations and actions that affect its observations, the optimal controller takes a linear form when the primitive random variables are Gaussian. In general however, this is not true for non-classical information structures other than the two structures above \citep{witsenhausen1968counterexample}. Moreover, it is known that the complexity of solving the optimal control in systems with non-classical information structure belongs to NEXP complexity class \citep{bernstein2002complexity}. Therefore, we believe that this assumption is reasonable as it preserves the challenges both in the form of the optimal controller and the complexity of solving it. 	

The assumption that each agent has access to the whole history can be relaxed by assuming access to a truncated history. Note that an LQR problem with  $\rho(A-B K)<1$ has the state-action pair sequence $\left(x_{t}, u_{t}\right)_{t \geq 0}$ that is $\beta$-mixing with parameter $\rho\in (\rho(A-B K),1)$. Therefore, the system forgets the history exponentially fast. Having access to the truncation history introduces an extra term in the optimality gap that bounds the estimation error of the natural gradient. This term will be small provided that the history is sufficiently long. As a result, the convergence of our algorithm is still guaranteed. We provide additional discussion in \S\ref{global_convergence}.

\subsection{Hierarchical Actor-Critic Algorithm}
\label{sec:hierarchical_ac}

In this section, we propose the hierarchical actor-critic algorithm. We start by using partial exchangeability to decompose the original optimal control problem in a multi-agent system into optimal control problems of $L+1$ auxiliary systems: $L$ for the subpopulations, denoted as $\{\mathcal{S}_l\}_{l\in \mathcal{L}}$, and one for the mean fields, denoted as $\bar{\mathcal{S}}$. The construction of auxiliary systems relies on a coordinate transformation. We also  define the cost function of each auxiliary system. We specify them in the following definitions.

\begin{definition}[Coordinate transformation]\label{coor_trans}
	For each agent $i \in \mathcal{N}^l$, we define the coordinate transformation as
	\[
	\tilde{x}_{t}^{i} = {x}_{t}^{i} - \bar{x}_{t}^{l},\qquad \tilde{u}_{t}^{i} = {u}_{t}^{i} - \bar{u}_{t}^{l}.
	\]
	The coordinate of auxiliary system $\mathcal{S}_l$ is defined to be the tuples 	$\tilde{\mathbf{x}}^l_t = (\tilde{x}_{t}^{i})_{i \in \mathcal{N}^l}$ and 	$\tilde{\mathbf{u}}^l_t = (\tilde{u}_{t}^{i})_{i \in \mathcal{N}^l}$. For the mean-field auxiliary system $\bar{\mathcal{S}}$, the coordinate is given by the mean-field values $\bar{\mathbf{x}}_{t}$ and $\bar{\mathbf{u}}_{t}$ defined in Definition~\ref{mean_field_states}.
\end{definition}

\begin{definition}[Cost functions]\label{cost_functions}
	The global total cost of $\mathcal{S}_l$ is given by 
	\[
	\tilde{c}^l =c_{gt}(\mathbf{x}_{t},\mathbf{u}_{t})- c_{gt}(\breve{\mathbf{x}}_{t}^{l},\breve{\mathbf{u}}_{t}^{l}),
	\]
	where $\breve{\mathbf{x}}_{t}^{l}$ denotes the joint state obtained by replacing each individual state $\{x^i_t\}_{i \in \mathcal{N}^{l}}$ in $\mathbf{x}_{t}$ with the mean-field states $\bar x_t^l$, and $\breve{\mathbf{u}}_{t}^{l}$ is defined similarly. With slight abuse of notation, $\breve{\mathbf{x}}_{t}^{l}$ and $\breve{\mathbf{u}}_{t}^{l}$ are defined, respectively, by
	\[
	\label{eq:breve_trans}
	\breve{\mathbf{x}}_{t}^{l}
	= \left(\text{rep}(\bar{x}_{t}^{l},|\mathcal{N}^l|),({x}_{t}^{j})_{j \in \mathcal{N}^{-l}}\right),
	\quad
	\breve{\mathbf{u}}_{t}^{l}
	= \left(\text{rep}(\bar{u}_{t}^{l},|\mathcal{N}^l|),({u}_{t}^{j})_{j \in \mathcal{N}^{-l}}\right),
	\]
	where we use $\text{rep}(\bar{x}_{t}^{l},|\mathcal{N}^l|)$ to denote the tuple obtained by replicating the vector $\bar{x}_{t}^{l}$ $|\mathcal{N}^l|$ times and we denote by $\mathcal{N}^{-l}$ all subpopulations other than $\mathcal{N}^{l}$. The cost of $\bar{\mathcal{S}}$ is given by 
	\[
	\bar{{c}} = c_{gt}(\breve{\mathbf{x}}_{t},\breve{\mathbf{u}}_{t}),
	\] 
	where $\breve{\mathbf{x}}_{t}$ is obtained by replacing all individual states in $\mathbf{x}_{t}$ with their corresponding subpopulation mean-field states, and $\breve{\mathbf{u}}_{t}$ is	defined similarly. In particular, $\breve{\mathbf{x}}_{t}$ and	$\breve{\mathbf{u}}_{t}$ are defined, respectively, by
	\begin{equation}
	\label{eq:breve_trans2}
	\breve{\mathbf{x}}_{t} = \left(\text{rep}(\bar{x}_{t}^{l},|\mathcal{N}^l|)\right)_{l\in \mathcal{L}},
	\quad \breve{\mathbf{u}}_{t} = \left(\text{rep}(\bar{u}_{t}^{l},|\mathcal{N}^l|)\right)_{l\in \mathcal{L}}.
	\end{equation}
\end{definition}

We remark that the state-action pairs of the auxiliary systems are induced by the state-action pairs of the original system through coordinate transformation, as is shown in Definition \ref{coor_trans}. The costs of the auxiliary systems are calculated with costs of the original system, as is shown in Definition \ref{cost_functions}. In the LQR problem, the costs of the auxiliary systems can be calculated directly using matrix computation. See Proposition~\ref{lemma:decpl}. On the other hand, the policies are defined in the auxiliary systems. After choosing actions with the states and policies of the auxiliary system, we can recover the states and policies of the original system and proceed with its transition dynamics.

For the auxiliary system $\mathcal{S}_l$, we assume that all agents share a common policy $\tilde\pi_{\theta_l}$ due to homogeneity within each subpopulation $\mathcal{N}^{l}$. Thus, it reduces to a single-agent system with state-action pairs $\{(\tilde{\mathbf{x}}^l_t,\tilde{\mathbf{u}}^l_t)\}_{t\geq 0}$ induced by $\tilde\pi_{\theta_l}$ and cost $\tilde{c}^l$ at time step $t$.  Agents in $\mathcal{S}_l$ aim to search for a common optimal policy that minimizes the corresponding expected time-average cost $\tilde C^l$. Similarly, $\bar{\mathcal{S}}$ is a single-agent system with state-action pairs $\{(\bar{\mathbf{x}}_{t}, \bar{\mathbf{u}}_{t})\}_{t\geq 0}$ induced by current policy $\bar\pi_{\bar\theta}$ and cost $\bar{{c}}$. The agent aims to search for an optimal policy that minimizes the corresponding expected time-average cost $\bar C$.

\begin{algorithm} [t]
	\caption{Hierarchical (Natural) Actor-Critic} 
	\label{algo:multi-agent} 
	\begin{algorithmic} 
		\STATE{{\textbf{Input:}} Number of iteration $N$, the partition $\{\mathcal{N}^{l}:l \in \mathcal{L} :=\{1, \ldots, L\}\}$, stepsizes $\{\eta_l:\ l \in L\}$ and $\bar\eta $.}
		\STATE{{\textbf{Initialization:}} Initialize policies $\{\tilde \pi_l^0\}_{l \in \mathcal{L}}$ and  $\bar\pi^0 $ for the auxiliary systems $\{\mathcal{S}_l\}_{l \in \mathcal{L}}$ and $\bar{\mathcal{S}}$ respectively. }
		\FOR{$n = 0,\dots, N$}
		\STATE{{\textbf{Critic step.}}}
		\STATE Initialize states $\mathbf{x}_{0}=\left(x_{0}^{i}\right)_{i \in \mathcal{N}}$, do coordinate transformation to obtain initial states $\{\tilde{\mathbf{x}}^l_0\}_{l\in \mathcal{L}}$ and $\bar{\mathbf{x}}_{0}$ of auxiliary systems. 
		\FOR{$t = 0,\dots, T$}
		\STATE Take actions $\{\tilde{\mathbf{u}}^l_t\}_{l\in \mathcal{L}}$  and $\bar{\mathbf{u}}_{t}$ in each auxiliary system respectively based on current states $\{\tilde{\mathbf{x}}^l_t\}_{l\in \mathcal{L}}$, $\bar{\mathbf{x}}_{t}$ and policies $\{\tilde \pi_l^n\}_{l \in \mathcal{L}}$,  $\bar\pi^n $.
		\STATE Recover the original coordinates $\mathbf{x}_{t}$ and $\mathbf{u}_{t}$. Calculate the costs $\tilde{c}^l$ and $\bar{{c}}_{t}$ based on the original cost $c_{gt}$. Observe the next state $\mathbf{x}_{t+1} = \mathcal{T} (\mathbf{x}_{t}, \mathbf{u}_{t})$.
		\STATE Do coordinate transformation to obtain the next auxiliary states $\{\tilde{\mathbf{x}}^l_{t+1}\}_{l\in \mathcal{L}}$ and $\bar{\mathbf{x}}_{t+1}$.
		\ENDFOR
		\STATE Obtain the estimators of the action-value functions $\{\hat Q_{\pi_l}^l\}_{l \in \mathcal{L}}$ and $\hat{\bar Q}_{\bar\pi}$ via a policy evaluation algorithm in the auxiliary systems. For online algorithms, the estimation is implemented during the simulation (for example, Algorithm \ref{algo:gtd} in \S\ref{section:gtd}). 
		\STATE{{\textbf{Actor step.}}  
			\STATE Update the auxiliary policies by the (natural) policy gradient decent 
			with gradients estimated by $\{\hat Q_{\pi_l}^l\}_{l \in \mathcal{L}}$ and $\hat{\bar Q}_{\bar\pi}$ and step sizes $\{\eta_l:\ l \in L\}$ and $\bar\eta $.}
		\ENDFOR
		\STATE{{\textbf{Output:}} The final policies $\{\tilde\pi_{l}^{N}\}_{l \in \mathcal{L}}$ and $\bar\pi^{N}$.}
	\end{algorithmic}
\end{algorithm}

The resulting action-value functions $\{\tilde Q_{\theta_l}^l\}_{l \in \mathcal{L}}$ and $\bar Q_{\bar\theta}$ are still coupled since the costs $\tilde{c}^l$, $\bar{\mathbf{c}}_{t}$ and the dynamics depend on the joint state $\mathbf{x}_{t}$ and action $\mathbf{u}_{t}$. We address this by assuming that for each auxiliary system, the action-value function has either a decoupled form or can be approximated by a decoupled function that only depends on the coordinates of that auxiliary system.   This assumption allows us to update policies separately. As we will see in the next section, this assumption is without loss of generality due to the notion of partial exchangeability for LQR problems. As LQR models are fundamental in approximating general dynamic systems, our method readily applies to a number of practical settings, such as Unmanned Aerial Vehicle (UAV) \citep{zhi2017flight,setyawan2019linear} and Power Grids \citep{minciardi2011optimal, vinifa2016linear}. The decomposition might not hold exactly for general dynamic systems, however, the empirical success of decentralized/decoupled methods in various applications \citep{omidshafiei2017deep, zhang2020trajectory} justifies our assumption that the action-value function can be approximated by a decoupled function.

Algorithm~\ref{algo:multi-agent} provides a summary of the hierarchical actor-critic algorithm. Essentially, we are evaluating and updating policies with the actor-critic algorithm in the auxiliary systems and observe the dynamic transitions and costs in the original system. We provide a rigorous justification of this algorithm in the LQR setting in \S\ref{Hierarchical_LQR} and establish the global convergence in \S\ref{global_convergence}.

\section{Main Results}
\label{section:main_result}

In \S\ref{Hierarchical_LQR} we provide a rigorous justification of Algorithm~\ref{algo:multi-agent} in the LQR setting. The hierarchical natural actor-critic algorithm for the LQR problem is specified in \S\ref{sec:multi-agent_LQR}.  In \S\ref{global_convergence} we establish the provable global convergence for the hierarchical natural actor-critic algorithm.

\subsection{Decomposition of LQR with Partial Exchangeability}
\label{Hierarchical_LQR}

We focus on the multi-agent LQR optimal control problem defined as
\begin{align}
\minimize_{K} C(K) &= \limsup _{T \rightarrow \infty} T^{-1} \sum_{t=0}^{T-1} \mathbb{E}_{\mathbf{x}_{0} \sim \mathcal{D}_{0}, \mathbf{u}_{t}  \sim \pi_K(\cdot|\mathbf{x}_{t})}\left[c_{gt}\left(\mathbf{x}_{t}, \mathbf{u}_{t}\right)\right]\notag\\
\text{subject to \quad} \mathbf{x}_{t+1}&=A  \mathbf{x}_{t}+B \mathbf{u}_{t}+\mathbf{w}_{t}, \quad c_{gt}\left(\mathbf{x}_{t}, \mathbf{u}_{t}\right)=\mathbf{x}_{t}^{\top} Q \mathbf{x}_{t}+\mathbf{u}_{t}^{\top} R \mathbf{u}_{t}
\label{eq:original_system},
\end{align}
where $\mathbf{x}_{t}=\vec\left(\left({x}_{t}^{i}\right)_{i \in \mathcal{N}} \right)$ and  $\mathbf{u}_{t}=\vec\left(\left({u}_{t}^{i}\right)_{i \in \mathcal{N}} \right)$ denote the global joint state and action at time $t$ of all agents. Recall that the optimal control takes a linear form and we use the  matrix $K$ to parameterize the policy $\pi$ such that  $\mathbf{u}_{t} = -K\mathbf{x}_{t}+\sigma \mathbf{z}_{t}$, $\mathbf{z}_{t}\sim N(0,I_d)$. In addition to being of fundamental importance, the LQR problem is frequently used in practice to approximate the original problem in \eqref{eq:general_dymcs}.

We focus on the case where the system satisfies partial exchangeability with partition $\mathcal{N} = \bigsqcup_l^L \mathcal{N}^l$. We show that after the coordinate transformation, optimization problem \eqref{eq:original_system} can be decomposed into $L+1$ control problems that correspond to the $L+1$ auxiliary systems $\{\mathcal{S}_l\}_{l \in \mathcal{L}}$ and $\bar{\mathcal{S}}$. This is established by proving that in auxiliary systems, the dynamics, costs and thus the action-value functions take decoupled forms. As a result, each of the control problems can be controlled separately with a linear policy. Finally, the observation that minimizing the decoupled objectives separately for all auxiliary systems decreases the global total expected time-average cost $C(K)$ concludes the validity of Algorithm~\ref{algo:multi-agent} in the LQR setting.

We first introduce Lemma \ref{lemma:dnms_cost}, adapted from \cite{arabneydi2016linear}, that expresses the agent's individual dynamics and cost with respect to the mean-field and individual state-action pairs. 

\begin{lemma}
	\label{lemma:dnms_cost}
	Suppose the LQR problem specified in \eqref{eq:original_system}  satisfies partial exchangeability with an exchangeable partition $\mathcal{N} = \{\mathcal{N}^{l}\}_{l \in \mathcal{L} :=[L]}$. There exist matrices $A_{l}$, $B_{l}$,  $\breve{A}_l$, $\breve{B}_l$,  $Q_{l}$, $R_{l}$,  $\breve{Q}$ and  $\breve{R}$, explicitly defined by ${A}$, ${B}$, ${Q}$ and ${R}$ with dimensions independent of the number of agents in each subpopulation, such that the individual dynamics and cost function in \eqref{eq:original_system} decompose as
	\begin{align}
	&x_{t+1}^{i}=A_{l} x_{t}^{i}+B_{l} u_{t}^{i}+\breve{A}_{l} \bar{\mathbf{x}}_{t}+\breve{B}_{l} \bar{\mathbf{u}}_{t}+w_{t}^{i},\label{second_dnmcs}\\
	&c\left(\mathbf{x}_{t}, \mathbf{u}_{t}\right)=\bar{\mathbf{x}}_{t}^{\top}  \breve{Q} \bar{\mathbf{x}}_{t}+\bar{\mathbf{u}}_{t}^{\top} \breve{R} \bar{\mathbf{u}}_{t}+\sum_{l \in \mathcal{L}} \sum_{i \in \mathcal{N}^{l}} \left[\left(x_{t}^{i}\right)^{\top} Q_{l} x_{t}^{i}+\left(u_{t}^{i}\right)^{\top} R_{l} u_{t}^{i}\right].\label{second_cost}
	\end{align}
\end{lemma}
The proof is given in \S\ref{pf:dnms_cost}, where we also provide explicit definitions of the matrices $A_{l}$, $B_{l}$, $\breve{A}_l$, $\breve{B}_l$, $Q_{l}$, $R_{l}$,  $\breve{Q}$ and $\breve{R}$.

Based on the matrices defined in Lemma~\ref{lemma:dnms_cost}, we further define the following matrices that turn out to be useful in defining the auxiliary systems. Specifically, we define matrices $\bar{A}$ and $\bar{B}$ as 
\begin{align*}
\bar{A} & :=\operatorname{diag}\left(A_{1}, \ldots, A_{L}\right)+\operatorname{rows}\left(\breve{A}_{1}, \ldots, \breve{A}_{L}\right), \\ 
\bar{B} & :=\operatorname{diag}\left(B_{1}, \ldots, B_{L}\right)+\operatorname{rows}\left(\breve{B}_{1}, \ldots, \breve{B}_{L}\right),
\end{align*}
and matrices $\bar{Q}$ and $\bar{R}$  as 
\[
\bar{Q} :=\breve{Q} + \operatorname{diag}\left(|\mathcal{N}^1|\cdot Q_{1}, \ldots, |\mathcal{N}^L|\cdot Q_{L}\right), \quad 
\bar{R} :=\breve{R} +\operatorname{diag}\left(|\mathcal{N}^1|\cdot R_{1}, \ldots,|\mathcal{N}^L| \cdot R_{L}\right).
\]
The following standard assumption \citep{fazel2018global,arabneydi2016linear} ensures that the cost functions of the auxiliary systems are well defined.
\begin{assumption}\label{assump:semi_posi}
	Matrices $\{Q_l\}_{l \in \mathcal{L}}$, $\bar{Q}$,  $\{R_l\}_{l \in \mathcal{L}}$ and $\bar{R}$ are positive definite.
\end{assumption}

The following proposition tells us that the dynamics and costs of the auxiliary systems have a decoupled form. Note that we also apply the coordinate transformation to the noise terms.

\begin{proposition}[Auxiliary systems with decoupled dynamics and costs]
	\label{lemma:decpl}
	Suppose the assumptions of Lemma~\ref{lemma:dnms_cost} and  	Assumption~\ref{assump:semi_posi} hold. After an application of the coordinate transformation \eqref{coor_trans}, the dynamics of the auxiliary systems $\mathcal{S}_l$ and $\bar{\mathcal{S}}$ induced by
	the original dynamics in \eqref{eq:original_system} can be written as
	\begin{align}
	\label{eq:decpl_dnms}
	\tilde{x}_{t+1}^{i}&=A_{l} \tilde{x}_{t}^{i}+ B_{l} \tilde{u}_{t}^{i}+\tilde{w}_{t}^{i}, \quad \tilde{w}_{t}^{i} \sim N(0, \Phi_l), \\
	\bar{\mathbf{x}}_{t+1}&=\bar{A} \bar{\mathbf{x}}_{t}+\bar{B} \bar{\mathbf{u}}_{t}+\bar{\mathbf{w}}_{t}, \quad \bar{\mathbf{w}}_{t} \sim N(0, \bar{\Phi}).
	\end{align}
	Furthermore, the global total costs of $\mathcal{S}_l$ and the cost of $\bar{\mathcal{S}}$, defined in Definition~\ref{cost_functions}, only depend on state-action pairs 	$(\tilde{\mathbf{x}}^l_t,\tilde{\mathbf{u}}^l_t)$ and $\left(\bar{\mathbf{x}}_{t}, \bar{\mathbf{u}}_{t}\right)$, respectively, and can be written as
	\begin{align}
	\label{eq:decpl_cost2}
	\tilde{c}^l(\tilde{\mathbf{x}}^l_t,\tilde{\mathbf{u}}^l_t)&=\sum_{i \in \mathcal{N}^{l}} \left[\left(\tilde{x}_{t}^{i}\right)^{\top} Q_{l} \tilde{x}_{t}^{i}+\left(\tilde{u}_{t}^{i}\right)^{\top} R_{l} \tilde{u}_{t}^{i}\right],\\
	\label{eq:decpl_cost1}
	\bar{c}\left(\bar{\mathbf{x}}_{t}, \bar{\mathbf{u}}_{t}\right)&=\bar{\mathbf{x}}_{t}^{\top}{\bar{Q} }
	\bar{\mathbf{x}}_{t}+\bar{\mathbf{u}}_{t}^{\top}{\bar{R} }\bar{\mathbf{u}}_{t}.
	\end{align}	
	Moreover, the original global total cost function decomposes as
	\begin{equation}
	\label{eq:decpl_cost}
	c_{gt}\left(\mathbf{x}_t, \mathbf{u}_{t}\right)=\bar{c}\left(\bar{\mathbf{x}}_{t}, \bar{\mathbf{u}}_{t}\right)+\sum_{l \in \mathcal{L}}\tilde{c}^l(\tilde{\mathbf{x}}^l_t,\tilde{\mathbf{u}}^l_t).
	\end{equation}
\end{proposition}

Note that by Proposition~\ref{lemma:decpl}, the original system decomposes into $L+1$ auxiliary systems: $L$ for each subpopulation $\{\mathcal{N}^{l}\}_{l \in \mathcal{L}}$, and one for the mean-field system. The auxiliary systems have decoupled dynamics and costs, hence they can be controlled with separate policies. Moreover, as is shown in \eqref{eq:decpl_cost2}, the individual cost (also denoted by $\tilde{c}^l(\cdot,\cdot)$ with a slight abuse of notation) in $\mathcal{S}_l$ takes the identical form
\[ 
\tilde{c}^l(\tilde{{x}}^i_t,\tilde{{u}}^i_t) = \left(\tilde{x}_{t}^{i}\right)^{\top} Q_{l} \tilde{x}_{t}^{i}+\left(\tilde{u}_{t}^{i}\right)^{\top} R_{l} \tilde{u}_{t}^{i}.
\] 
 Therefore, agents in the subpopulation $\mathcal{N}^l$ share the matrices $A_l$, $B_l$ for the dynamics and $Q_l$, $R_l$ for the costs. As a result, their optimal policies are identical, which justifies the usage of a common policy $\tilde\pi_l$ in Algorithm~\ref{algo:multi-agent}.

We parameterize the policies of auxiliary systems by matrices $ K_l \in \mathbb{R}^{k_l \times d_l}$ and $\bar K \in \mathbb{R}^{\overline k \times \overline d}$. By adding the Gaussian noise to allow for exploration, the policies can be written as 
\begin{equation}
\begin{aligned}
\label{eq:pi_K_pi_K_mean}
\tilde u_{t}^i&=- K_l \tilde x_{t}^i+\sigma_l \cdot \tilde z_{t}^i, \quad \tilde z_{t}^i \sim N\left(0, I_{d_l}\right),\ i \in \mathcal{N}^l,\ l\in \mathcal{L},\\
\bar{\mathbf{u}}_{t}&=-\bar K \bar{\mathbf{x}}_{t}+\bar\sigma \cdot \bar{\mathbf{z}}_{t}, \quad \bar z_{t} \sim N\left(0, I_{\bar d}\right), 
\end{aligned}
\end{equation}
with the corresponding distributions denoted as $\pi_{ K_l}(\cdot \mid \tilde x_t^i)$ and $\pi_{\bar K}(\cdot \mid \bar{\mathbf{x}}_{t})$. Our next result states that minimizing the original objective $C(K)$ can be done separately with respect to $\bar K$ and $ K_l$, $l \in \mathcal{L}$.
\begin{proposition}
	\label{prop:ergodic_cost_decompositon}
	The objective $C(K)$ can be decomposed as $C(K) = \bar C(\bar K) + \sum_{l \in \mathcal{L}}  \tilde C(\tilde K_l)$, where
	\begin{align}
	\bar C(\bar K) &:= \lim _{T \rightarrow \infty} \mathbb{E}_{\bar{\mathbf{x}}_{0} \sim \bar{\mathcal{D}}_{0}, \bar{\mathbf{u}}_{t} \sim \pi_{\bar K}\left(\cdot | \bar{\mathbf{x}}_{t}\right)}\left[\frac{1}{T} \sum_{t=0}^{T-1}	\bar{c}\left(\bar{\mathbf{x}}_{t}, \bar{\mathbf{u}}_{t}\right)\right], \\
	\tilde C( K_l) & := \lim _{T \rightarrow \infty} \mathbb{E}_{\tilde x_0^i \sim \tilde{\mathcal{D}}_{0}^l, \tilde u_{t}^i \sim \pi_{\tilde K_l}\left(\cdot | \tilde x_{t}^i\right)} \left[\frac{1}{T} \sum_{t=0}^{T-1}\tilde{c}^l(\tilde{\mathbf{x}}^l_t,\tilde{\mathbf{u}}^l_t)\right],\ i \in \mathcal{N}^l, \ l\in \mathcal{L}.
	\end{align}
\end{proposition}
The result follows by direct computation and is proved in \S\ref{pf:ergodic_cost_decompositon}. Note that $\bar C(\bar K)$ and $\tilde C( K_l)$ are exact objectives of the optimal control problems defined in the auxiliary systems and minimizing them separately yields the minimum of the original objective $C(K)$. With the decoupled dynamics and cost functions, as well as the linearly parameterized policies described above, the corresponding action-value functions $\{\widetilde{Q}_{K_{l}}\}_{l \in \mathcal{L}}$ and $\bar{Q}_{\bar{K}}$ indeed have decoupled structures, which justifies Algorithm~\ref{algo:multi-agent} in the LQR setting. 

\subsection{Hierarchical Natural Actor-Critic in LQR Setting}
\label{sec:multi-agent_LQR}

In this section, we develop the hierarchical natural actor-critic algorithm for the LQR problem. With the action distribution defined in \eqref{eq:pi_K_pi_K_mean}, the state dynamics
defined by \eqref{eq:decpl_dnms} take respectively the forms
\begin{equation}
\begin{aligned}
\tilde x_{t+1}^i&=( A_l - B_l  K_l) \tilde x_{t}^i+\tilde\varepsilon_{t}^i, \quad \tilde\varepsilon_{t}^i\sim N\left(0, \Phi_{\sigma_l}^l\right),\ i \in \mathcal{N}^l,\ l\in \mathcal{L}, \\
\bar x_{t+1}&=(\bar A-\bar B \bar K) \bar x_{t}+ \bar \varepsilon_{t}, \quad \bar \varepsilon_{t}\sim N\left(0, \bar \Phi_{\bar\sigma}\right),\label{markov_chain_}
\end{aligned}
\end{equation}
where $\Phi_{\sigma_l}^l := \Phi_l+\sigma_l^{2} \cdot  B_l  B_l^{\top}$ and $\bar \Phi_{\bar\sigma} := \bar \Phi+\bar\sigma^{2} \cdot \bar B \bar B^{\top}$.  Let $\{\Sigma_{K_l}\}_{l \in \mathcal{L}}$ and $\Sigma_{\bar K}$ denote the unique positive definite solutions to the Lyapunov equations 
\begin{align*}
\Sigma_{K_l}&=\Phi_{\sigma_l}^l+(A_l-B_l K_l) \Sigma_{K_l}(A-B K_l)^{\top},\\
\Sigma_{\bar K}&=\bar\Phi_{\bar\sigma}+(\bar A-\bar B \bar K) \Sigma_{\bar K}(\bar A-\bar B \bar K)^{\top}. 
\end{align*}
Under the condition that $\rho(A_l-B_l K_l)<1$ and $\rho(\bar A- \bar B \bar K)<1$, the Markov chains introduced by \eqref{markov_chain_} have stationary distributions $N\left(0, \Sigma_{K_l}\right)$ and $N\left(0, \Sigma_{\bar K}\right)$, denoted by $\mathcal{D}_{K_l}$ and ${\mathcal{D}}_{\bar K}$, respectively.

The following lemma establishes the functional forms of costs $\tilde C( K_l)$ and $\bar C(\bar K)$, as well as their gradients.
\begin{lemma}\label{lemma:cost_grad_K}
	The ergodic costs are given by
	\begin{align}\label{eq:cost_K2}
	\tilde C( K_l)&=\tr\left[\left(Q_l+K^{\top}_l R_l K_l\right) \Sigma_{K_l}\right]+\sigma_l^{2} \cdot \tr(R_l)\notag\\&=\tr\left(P_{K_l} \Phi_{\sigma_l}^l\right)+\sigma_l^{2} \cdot \tr(R_l),\\
	\bar C(\bar K)&=\tr\left[\left(\bar{Q}+\bar K^{\top}\bar{R} \bar K\right) \Sigma_{\bar K}\right]+\bar\sigma^{2} \cdot \tr(\bar{R})\notag\\
	&=\tr\left( P_{\bar K} \bar\Phi_{\bar\sigma}\right)+\bar\sigma^{2} \cdot \tr(\bar{R}),
	\end{align}
	with gradients 
	\begin{align}
	\nabla_{K_l} \tilde C(K_l)&= 2\left[\left(R_l+B_l^{\top} P_{K_l} B_l\right) K_l-B_l^{\top} P_{K_l} A_l\right] \Sigma_{K_l}=2 E_{K_l} \Sigma_{K_l},\\
	\nabla_{\bar K} \bar C(\bar K)&=2\left[\left(\bar{R}+B^{\top}  P_{\bar K} B\right) \bar K-B^{\top}  P_{\bar K} A\right]  \Sigma_{\bar K}=2  E_{\bar K}  \Sigma_{\bar K},
	\end{align}
	where $P_{K_l}$ and $P_{\bar K}$ are obtained as the solution to
	\begin{align*}
	P_{K_l}&=\left(Q_l+K_l^{\top} R_l K_l\right)+(A_l-B_l K_l)^{\top} P_{K_l}(A_l-B_l K_l),\\ 	
	P_{\bar K}&=\left(\bar{Q}+\bar K^{\top} \bar{R} \bar K\right)+(\bar A-\bar B \bar K)^{\top} P_{\bar K}(\bar A-\bar B \bar K),
	\end{align*}
	and $E_{K_l}$ and $E_{\bar K}$ are defined as
	\begin{align*}
	E_{K_l} &= \left(R_l+B_l^{\top} P_{K_l} B_l\right) K_l-B_l^{\top} P_{K_l} A_l,\\
	E_{\bar K} &= \left(\bar{R}+\bar B^{\top} P_{\bar K} \bar B\right) \bar K-\bar B^{\top} P_{\bar K} \bar A.
	\end{align*}
\end{lemma}
The lemma directly follows from the functional forms of the cost functions and gradients of a general LQR problem applied to each auxiliary system. The proof is presented in \S\ref{sec:supp_pf}.

To see how the natural policy gradient is related to $E_{K_l}$, observe that $\mathcal{I}(K_l)$ has block diagonal structure with $k_l$ blocks of size $d_l\times d_l$. Each block contains entries with coordinates of the form $(i,\cdot)\times(i,\cdot)$, where $i\in\{1\dots k_l\}$ and all of the blocks are identical to $\Sigma_{K_l}$. Hence, the natural policy gradient algorithm updates the policy in the direction of 
\[
[\mathcal{I}(K_l)]^{-1} \nabla_{K_l} \tilde{C}(K_l)=\nabla_{K_l} \tilde{C}(K_l) \Sigma_{K_l}^{-1}=2 E_{K_l}.
\] 
Similarly, the natural gradient for mean-field system is $2 E_{\bar K}$. In the critic step, the model-free estimates $\{\hat E_{K_l}\}_{l \in \mathcal{L}}$ and $\hat E_{\bar{K}}$ can be obtained with an online gradient-based temporal-difference algorithm \citep{sutton2009fast}. In the actor step, the policies are updated with $\{\hat E_{K_l}\}_{l \in \mathcal{L}}$ and $\hat E_{\bar{K}}$. Thus, we obtain the hierarchical natural actor critic algorithm for the LQR problem, as is summarized in Algorithm~\ref{algo:multi-agent_LQR}.

\begin{algorithm} [tbp]
	\caption{Hierarchical Natural Actor-Critic for LQR Problem} 
	\label{algo:multi-agent_LQR} 
	\begin{algorithmic} 
		\STATE{{\textbf{Input:}} Number of iteration $N$, the partition $\{\mathcal{N}^{l}:l \in \mathcal{L} :=\{1, \ldots, L\}\}$, stepsizes $\{\eta_l:\ l \in L\}$ and $\bar\eta $.}
		\STATE{{\textbf{Initialization:}} Initialize policies $\{\pi_{K_{l,0}}:\ l \in L\}$ and  $\bar\pi_{K_0}$. }
		\FOR{$n = 0,\dots, N$}
		\STATE{\textbf{Critic step.} Simulate current policies $\{ \pi_{K_{l,n}}\}_{l \in \mathcal{L}}$ and $\bar\pi_{K_n} $ following the same steps in Algorithm~\ref{algo:multi-agent}. Update estimators $\{\hat E_{K_{l,n}}:\ l \in L\}$ and $\hat E_{\bar K_n}$ via online GTD algorithm (Algorithm \ref{algo:gtd} in \S\ref{section:gtd}) at each simulating iteration.}

		\STATE{{\textbf{Actor step.}}  Update the policy parameter by  
			$ K_{l, n+1} =  K_{l, n} - \eta_l \cdot  \widehat{E}_{K_{l,n}}   $ and $\bar K_{n+1} =  \bar K_n - \bar\eta \cdot  \widehat{E}_{\bar K_n}   $.}
		\ENDFOR
		\STATE{{\textbf{Output:}} The final policies $\{\pi_{K_{l,N}}:\ l \in L\}$ and $\bar\pi_{K_{N}}$.}
	\end{algorithmic}
\end{algorithm}	

\subsection{Global Convergence}
\label{global_convergence}

In this section, we prove that the hierarchical natural actor critic algorithm, described in the previous section, converges globally to the optimal policy at a linear rate for LQR problems. We start by introducing notation and making some mild assumptions. At iteration $n$, the algorithm produces auxiliary policies $\{K_{l,n}\}_{l \in \mathcal{L}}$ and $\bar K_n$ for subpolulations and the mean fields, respectively.  They induce policy $K_n$ for the multi-agent system. We denote by $\{K^*_l\}_{l \in \mathcal{L}}$ and $\bar K^*$ the corresponding optimal policies. They induce the optimal policy for the multi-agent system, denoted by $K^*$. With this notation, we make the following assumptions.

\begin{assumption}
	\label{assump:stable_initial_policy}
	The initial policies $\{\pi_{K_{l,0}}:\ l \in \mathcal{L}\}$ satisfy $\rho\left(A_{l}-B_{l} K_{l,0}\right)<1$ and 
	$\bar\pi_{K_0}$ satisfies $\rho(\bar{A}-\bar{B} \bar{K}_0)<1$.
\end{assumption}

\begin{assumption}
	\label{assump:stepsize}
	The stepsizes are sufficiently small and satisfy
	\begin{align*}
	\eta_l &\leq\left[\|R_l\|+\sigma_{\min }^{-1}(\Phi_l) \cdot\|B_l\|^{2} \cdot C\left(K_{l,0}\right)\right]^{-1}, \forall l \in \mathcal{L},\\ \bar\eta&\leq\left[\|\bar R\|+\sigma_{\min }^{-1}(\bar \Phi) \cdot\|\bar B\|^{2} \cdot C\left(\bar K_{0}\right)\right]^{-1}.
	\end{align*}
\end{assumption}

\begin{assumption}
	\label{assump:pe}
	The estimates of natural gradients given by the critic step satisfy 
	\[
	\left\|\widehat{E}_{K_{l,n}}-E_{K_{l,n}}\right\| \leq \delta_{l,n}
	\qquad\text{and}\qquad \left\|\widehat{E}_{\bar K_n}-E_{\bar K_n}\right\| \leq \bar\delta_n, 
	\]
	where $\{\delta_{l,n}\}_{l \in \mathcal{L}}$ and $\bar\delta_n$ are sufficiently small positive values satisfying
	\begin{align*}
	\delta_{l,n} &\leq  \epsilon/2 \cdot\eta_l \cdot \sigma_{\min }(\Phi_l) \cdot \sigma_{\min }(R) \cdot\left\|\Sigma_{K^*_l}\right\|^{-1}\cdot 1/\Lambda(\|K_{l,n}\|,C(K_{l,0}))\cdot1/(L+1), \\
	\bar\delta_n& \leq  \epsilon/2 \cdot\bar\eta \cdot \sigma_{\min }(\bar\Phi) \cdot \sigma_{\min }(R) \cdot\left\|\Sigma_{\bar K^*}\right\|^{-1}\cdot 1/\Lambda(\|\bar K_n\|,C(\bar K_0))\cdot1/(L+1).
	\end{align*}
	Here $\Lambda(\cdot,\cdot)$ is a polynomial and $\epsilon$ is the error level we want to achieve, that is, $C\left(K_{N}\right)-C\left(K^{*}\right) \leq \epsilon$.
\end{assumption}

 Assumption \ref{assump:stable_initial_policy} is a standard assumption for the model-free LQR problem \citep{fazel2018global,malik2018derivative,tu2017least},
 despite the concern that finding such a stable policy without prior knowledge of the system parameters is difficult \citep{lewis2009reinforcement, bu2019lqr}. One simple way to address this problem is by approximating the infinite horizon problem with a finite horizon one \citep{fazel2018global}. Also, recently \citet{perdomo2021stabilizing} shows that an unknown dynamical system can be stabilized efficiently via a model-free policy gradient method. Assumptions \ref{assump:stepsize} and \ref{assump:pe} are technical assumptions on the relative updating steps of the natural gradient decent in the actor step and the GTD algorithm in the critic step (see Algorithm~\ref{algo:gtd} proposed in appendix). In particular, Assumption \ref{assump:pe}  states that the critic is updated at a faster pace than the actor. Under these assumptions we can prove non-asymptotic convergence results in contrast to asymptotic results of classical actor-critic algorithms \citep{konda2000actor,bhatnagar2009natural,grondman2012survey}.  
Note that Assumption \ref{assump:pe} is rather weak and can be satisfied by setting the number of iterations in the GTD algorithm sufficiently large. The statistical rate of convergence for the GTD algorithm is established by the following theorem.
\begin{theorem}[Informal]
	Let $\widehat{E}_{K}$ be the output of the GTD algorithm (see Algorithm~\ref{algo:gtd}) for policy $\pi_K$ with $T$ iterations.  For a sufficiently large $T$, we have
	\#\label{eq:informal_estimate_error} 
		\| \widehat{E}_{K} -E_{K}\|_{2}^2 \leq \Upsilon \cdot \frac{{\log^6 T}}{ \sqrt{T}},
	\#
	with probability at least $1 - T^{-4}$,	where $\Upsilon$ is a polynomial of the system parameters.
\end{theorem} 
For an LQR problem with partial exchangeability, the polynomial $\Upsilon$ for all auxiliary systems jointly determines the statistical rate of convergence of the GTD algorithm for the original LQR problem. The polynomial $\Upsilon$ reflects the essential difficulty of estimating the system parameters of the auxiliary systems. Note that the dimension of the mean-field system is proportional to the number of subpopulations $L$, which characterizes the intrinsic difficulty of learning the multi-agent system. In a special case when $L =1$, the problem reduces to the setting of homogeneous agents. The complexity is the same as that of learning a single-agent system. If $L = |\mathcal{N}|$, all agents are heterogeneous and we cannot explore any symmetry. The complexity is the same as that of treating the problem as a huge LQR problem and directly learning it, which becomes problematic as the number of agents grows.

If the agents only have access to a truncated history of data, there will be an extra error term in the estimation error $\| \widehat{E}_{K} -E_{K}\|_{2}^2$.  The thresholds $\delta_{l,n}$ and $\bar\delta_n$ in Assumption 4 are essentially of the order $O(1)$. Therefore, the truncation will not affect the convergence analysis provided that the history is sufficiently long, but finite, since even with the extra estimation error term, the total error is still smaller than the corresponding threshold. This allows us to relax the assumption on the access to the whole history of data to the access to the truncated history of data. As a result, we can ensure memory efficiency.

Next, we present the main theorem.
\begin{theorem}[Global convergence of hierarchical actor-critic]\label{thm:multi_ac}
	Let Assumptions \ref{assump:stable_initial_policy}-\ref{assump:pe} hold. Then $\left\{C\left(K_{n}\right)\right\}_{n \geq 0}$ is a monotonously decreasing sequence. 	Moreover, for any $\epsilon>0$, if 	$N>2M\log\left\{\left[C(K_0)-C(K^*)\right]\cdot(L+1)/\epsilon\right\}$ where the constant $M$ is defined by
	\#
	M = \max\left\{\left\{\|\Sigma_{K^{*}_l}\|\cdot\eta^{-1}_l\cdot\sigma^{-1}_{\min }(\Phi_l)\cdot\sigma^{-1}_{\min }(R_l)\right\}_{l \in \mathcal{L}}, \left\|\Sigma_{\bar K^{*}}\right\|\cdot\bar\eta^{-1}\cdot\sigma^{-1}_{\min }(\bar\Phi)\cdot\sigma^{-1}_{\min }(\bar R)\right\}\notag,
	\#
	then
	$C\left(K_{N}\right)-C\left(K^{*}\right) \leq \epsilon$.
\end{theorem}

Theorem \ref{thm:multi_ac} establishes that Algorithm \ref{algo:multi-agent_LQR} converges globally to the optimal policy $K^*$ at a linear rate. Moreover, note that our algorithm involves $L+1$ decoupled optimal control problems whose complexity does not depend on the number of agents in each subpopulation. This feature and Theorem~\ref{thm:multi_ac} together guarantee the computational efficiency of our algorithm, allowing us to escape the curse of dimensionality.
In the next subsection, we prove our main result. 

\subsection{Proof of Theorem~\ref{thm:multi_ac}}
\label{sec:main_proof}

By Propositions \ref{lemma:decpl} and \ref{prop:ergodic_cost_decompositon}, the original LQR problem defined in \eqref{eq:original_system} decomposes into optimal control problems for each auxiliary system. Therefore, it is sufficient to prove the global convergence for each auxiliary system. Note that, since agents in $\mathcal{S}_l$ share a common policy $\pi_{K_{l}}$ in the algorithm, the optimal control problem reduces to the single agent case. In particular, we need to prove the convergence theorem for a single agent LQR problem. In the rest of this section, we no longer distinguish each LQR problem and remove all notation  that indicates the subpopulations and the mean field. We first present the global convergence result of the hierarchical natural actor-critic algorithm for the single agent LQR problem.

\begin{theorem}[Global convergence of actor-critic algorithm]\label{thm:ac}
	Suppose  the initial policy ${\pi}_{K_{0}}$ satisfies 
	\#\label{eq:single_assump1}
	\rho({\pi}_{K_{0}}) < 1,
	\# 
	the stepsize satisfies 
	\#\label{eq:single_assump2}
	{\eta} \leq\left[\|{R}\|+\sigma_{\min }^{-1}({\Phi}) \cdot\|{B}\|^{2} \cdot C\left({K}_{0}\right)\right]^{-1},
	\# 
	and that the estimate of the natural gradient given by the critic step satisfies 
	\#\label{eq:single_assump3}
	\left\|\widehat{E}_{{K}_{n}}-E_{{K}_{n}}\right\| \leq {\delta}_{n},
	\#
	where $\delta_n$ is a positive value satisfying
	\#\label{eq:single_assump3_}
	\delta_n& \leq   \epsilon/2 \cdot\eta \cdot \sigma_{\min }(\Phi) \cdot \sigma_{\min }(R) \cdot\left\|\Sigma_{ K^*}\right\|^{-1}\cdot 1/\Lambda(\| K_n\|,C( K_0))
	\#
	and $\Lambda(\cdot,\cdot)$ is a polynomial.
	Then $\left\{C\left(K_{n}\right)\right\}_{n \geq 0}$ is a monotonously decreasing sequence. Moreover, for any $\epsilon>0$, if the iteration number $N$ is large enough such that
	\[
	N > 2\left\|\Sigma_{K^{*}}\right\|\cdot\eta^{-1}\cdot\sigma^{-1}_{\min }(\Phi)\cdot\sigma^{-1}_{\min }(R)\cdot\log\left\{\left[C(K_0)-C(K^*)\right]/\epsilon\right\},
	\]
	we have $C\left(K_{N}\right)-C\left(K^{*}\right) \leq \epsilon$.
\end{theorem}

We remark that by setting the number of iterations in the GTD Algorithm \ref{algo:gtd} sufficiently large, we can make $\delta_n$ sufficiently small so that \eqref{eq:single_assump3_} is satisfied. Using Theorem~\ref{thm:ac}, Theorem \ref{thm:multi_ac} directly follows from Propositions \ref{lemma:decpl} and \ref{prop:ergodic_cost_decompositon}. In the reminder of the section, we prove Theorem~\ref{thm:ac}. 

Our proof can be decomposed into three steps. In the first step, we study the geometry of $C(K)$ as a function of $K$. In general, natural gradient descent methods are not guaranteed to converge to the global optimal due to non-convexity of the LQR optimization problem. However, the geometric condition called \emph{gradient domination}  \citep{fazel2018global} in the LQR setting helps us  prove convergence. In the second step, we show that the policy is improved at a linear rate along the direction of the oracle natural policy gradient at each iteration. In the third step, we show that the policy updated with the estimated natural policy gradient has a cost close to that of the policy updated with the oracle natural policy gradient, thus we can show linear convergence for it as well.

The following lemma establishes the gradient domination condition.

\begin{lemma}[Gradient domination] \label{lemma:dom}
	Let $K^{*}$ be an optimal policy for agents in $\mathcal{N}$. 
	Suppose $K$ has finite cost in the sense that $\rho(A-B K)<1$. Then it holds that
	\begin{multline}
	\label{eq:dom_l}
	\sigma_{\min }(\Phi) \cdot\left\|R+B^{\top} P_{K} B\right\|^{-1} \cdot \tr\left(E_{K}^{\top} E_{K}\right)
	\leq
	C(K)-C\left(K^{*}\right)  \\
	\leq 1 / \sigma_{\min }(R) \cdot\left\|\Sigma_{K^{*}}\right\| \cdot \tr\left(E_{K}^{\top} E_{K}\right).
	\end{multline}
\end{lemma}

Note that the upper bound in \eqref{eq:dom_l} takes the form $ 1 / \sigma_{\min }(R) \cdot\left\|\Sigma_{K^{*}}\right\| \cdot \left\langle E_{K}, E_{K}\right\rangle$. Therefore, updating the policy $K$ with the natural gradient $E_K$ in the actor step of Algorithm \ref{algo:multi-agent_LQR} minimizes the upper bound of the difference $C(K)-C\left(K^{*}\right)$. Moreover, the natural gradient will not vanish before reaching the optimum. The following lemma shows that the policy is improved at a linear rate along the direction of the true natural policy gradient, provided that the step size is small enough.

\begin{lemma}\label{lemma:exactgradient}
	Suppose \eqref{eq:single_assump1} and \eqref{eq:single_assump2} hold.
	Let $\{K_{n}\}$ be the sequence induced by the natural policy gradient algorithm started at the initial policy $K_0$. Let $K^\prime_{n+1} =K_{n}-\eta \cdot E_{K_n} $ be a single update along the direction of the true natural policy gradient. Then we have
	\begin{align}\label{eq:exactcostreduction}
	C\left(K_{n+1}^{\prime}\right)-C\left(K_{n}\right)&\leq  -\eta \cdot \sigma_{\min }(\Phi) \cdot \sigma_{\min }(R) \cdot\left\|\Sigma_{K^{*}}\right\|^{-1} \cdot\left[C\left(K_{n}\right)-C\left(K^{*}\right)\right]
	\end{align}
	and 
	\begin{align}\label{eq:linear_in_exact}
	C\left(K_{n+1}^{\prime}\right)-C\left(K^{*}\right)&\leq \left[1 -\eta \cdot \sigma_{\min }(\Phi) \cdot \sigma_{\min }(R) \cdot\left\|\Sigma_{K^{*}}\right\|^{-1} \right]\cdot\left[C\left(K_{n}\right)-C\left(K^{*}\right)\right].
	\end{align}
\end{lemma}

To draw a similar conclusion on the update $K_{t+1}=K_{t}-\eta \cdot \widehat{E}_{K t}$, we need to link the objectives $C(K_{t+1}^\prime)$ and $C(K_{t+1})$. The following lemma bounds the difference between $C\left(K_{n+1}\right)$ and $C\left(K_{n+1}^{\prime}\right)$ by problem parameters.

\begin{lemma} 
\label{lemma:bound_update_diff}
	Suppose \eqref{eq:single_assump1} - \eqref{eq:single_assump3_} hold. Furthermore, suppose $C(K_n)<C(K_0)$. Let $K_{n+1}^{\prime}=K_{n}-\eta \cdot E_{K_n}$ and $K_{n+1}=K_{n}-\eta \cdot \widehat{E}_{K_n}$ be updates along the exact and estimated natural policy gradient at time $t$. Then, for any fixed $\epsilon > 0$, it holds that
	\begin{align}\label{costdiff}
	\left|C\left(K_{n+1}\right)-C\left(K_{n+1}^{\prime}\right)\right| \leq \epsilon / 2\cdot \eta \cdot \sigma_{\min }(\Phi) \cdot \sigma_{\min }(R) \cdot\left\|\Sigma_{K^{*}}\right\|^{-1}.
	\end{align}
\end{lemma} 

When $C\left(K_{n}\right)-C\left(K^{*}\right) \geq \epsilon$, combining \eqref{eq:exactcostreduction} and \eqref{costdiff}, we have
\[
C\left(K_{n+1}\right)-C\left(K_{n}\right) \leq-\epsilon / 2 \cdot \eta \cdot \sigma_{\min }(\Phi) \cdot \sigma_{\min }(R) \cdot\left\|\Sigma_{K^{*}}\right\|^{-1}<0.
\] 
This shows that $\left\{C\left(K_{n}\right)\right\}_{n \geq 0}$ is monotonically decreasing. Moreover, combining \eqref{eq:linear_in_exact} with \eqref{costdiff}, we further conclude that 
\begin{equation*}
C\left(K_{n+1}\right)-C\left(K^*\right)\leq \left[ 1- \eta /2\cdot \sigma_{\min }(\Phi) \cdot \sigma_{\min }(R) \cdot\left\|\Sigma_{K^{*}}\right\|^{-1} \right] \left[C\left(K_{n}\right)-C\left(K^{*}\right)\right],
\end{equation*}
which shows a linear convergence in terms of the policy parameter. By direct computation, if the iteration number $N$ is large enough such that
\[
N>2\left\|\Sigma_{K^{*}}\right\|\cdot\eta^{-1}\cdot\sigma^{-1}_{\min }(\Phi)\cdot\sigma^{-1}_{\min }(R)\cdot\log\left\{\left[C(K_0)-C(K^*)\right]/\epsilon\right\},
\]
it holds that $C\left(K_{N}\right)-C\left(K^{*}\right) \leq \epsilon$. This concludes the  proof of Theorem \ref{thm:ac}.

\section{Technical Proofs}
\label{section:proof_sketch}

In this section, we present the proofs of our technical results in \S\ref{section:main_result}. Proofs of the supporting lemmas are deferred to Appendix~\S\ref{sec:supp_pf}.

\subsection{Proof of Lemma~\ref{lemma:dnms_cost}}
\label{pf:dnms_cost}

Recall that the dynamics and cost function of the LQR problem specified in \eqref{eq:original_system} are given as
\begin{align}
\mathbf{x}_{t+1}=A  \mathbf{x}_{t}+B \mathbf{u}_{t}+\mathbf{w}_{t}, \quad c_{gt}\left(\mathbf{x}_{t}, \mathbf{u}_{t}\right)=\mathbf{x}_{t}^{\top} Q \mathbf{x}_{t}+\mathbf{u}_{t}^{\top} R \mathbf{u}_{t},
\end{align}
and satisfy the partial exchangeability with exchangeable partition $\mathcal{N} = \{\mathcal{N}^{l}\}_{l \in \mathcal{L} :=[L]}$.

In the following, let $A^{i, j}$ denote the $(i,j)$-th block of $A$. Fix a subpopulation $l \in \mathcal{L}$. For agents $i, j \in \mathcal{N}^{l}$, the exchangeability in Definition~\ref{def:exchangeability} implies $A^{i, i}=A^{j, j}$ and $B^{i, i}=B^{j, j}$, denoted by $a^{l}$ and $b^{l}$, respectively. For $i, j \in \mathcal{N}^{l}$ and $n, m \in \mathcal{N}^{k}, l \neq k$, we have $A^{i, n}=A^{j, m}$ and $B^{i, n}=B^{j, m}$, denoted by $\breve{a}^{l, k}$ and $\breve{b}^{l, k}$, respectively. For $i, j \in \mathcal{N}^{l}$, we have $Q^{i, i}=Q^{j, j}$ and $R^{i, i}=R^{j, j}$, denoted by $q^{l}$ and $r^{l}$, respectively. For $i, j \in \mathcal{N}^{l}$ and $n, m \in \mathcal{N}^{k}, l \neq k$, we have $Q^{i, n}=Q^{j, m}$ and $R^{i, n}=R^{j, m}$, denoted by $\breve{q}^{l, k}$ and $\breve{p}^{l, k}$, respectively. With this notation we provide the explicit forms for $A_{l}$, $B_{l}$, $\bar{A}_l$, $\bar{B}_l$, $Q_{l}$, $R_{l}$, $\bar{Q}$ and $\bar{R}$. 

First, we define $A_{l}$ and $B_{l}$ as
\[
A_{l}:=a^{l}-\breve{a}^{l, l}, \qquad
B_{l}:=b^{l}-\breve{b}^{l, l}.
\]
We also define $\breve{A}_l$, $\breve{B}_l$, $Q_{l}$, and $R_{l}$ as
\begin{gather*}
\breve A_l  := \operatorname{cols}\left(\left|\mathcal{N}^{1}\right| \breve{a}^{l, 1},\dots, \left|\mathcal{N}^{L}\right| \breve{a}^{l, L}\right), \qquad
\breve B_l  := \operatorname{cols}\left(\left|\mathcal{N}^{1}\right| \breve{b}^{l, 1},\dots, \left|\mathcal{N}^{L}\right| \breve{b}^{l, L}\right), \\
Q_{l}:=\left(q^{l}- \breve{q}^{l, l}\right),  \qquad R_{l}:=\left(r^{l}-\breve{r}^{l, l}\right).
\end{gather*}
Finally, we define $\breve{Q}$ and $\breve{R}$ by specifying each block as
\[
\breve{Q}^{l, k}:=\left|\mathcal{N}^{l} \right| \left|\mathcal{N}^{k}\right| \breve{q}^{l, k}, \qquad 
\breve{R}^{l, k}:=\left|\mathcal{N}^{l}\right|\left|\mathcal{N}^{k}\right| \breve{r}^{l, k}. 
\]
We remark that the dimensions of $A_{l}$, $B_{l}$, $\breve{A}_l$, $\breve{B}_l$, $Q_{l}$, $R_{l}$,  $\breve{Q}$, and $\breve{R}$ do not depend on the size of each subpopulation. Instead, they are determined by the subpopulations' state- and action- dimensions.

With the definitions above, we are ready to present the proof. The dynamics of agent $i$ of subpopulation $k$ is 
\[
x_{t+1}^{i}=A^{i \cdot} \mathbf{x}_{t}+B^{i \cdot} \mathbf{u}_{t}+w_{t}^{i},
\]
where $A^{i \cdot}$ and $B^{i \cdot}$ denote, respectively, the rows corresponding to the $i$-th column blocks of $A$ and $B$. By direct computation, we have
\begin{align*}
A^{i \cdot} \mathbf{x}_{t} 
&=A^{i, i} x_{t}^{i}+\sum_{j \in \mathcal{N}^{l}, j \neq i} A^{i, j} x_{t}^{j}+\sum_{k \in \mathcal{L}, k \neq l} \sum_{n \in \mathcal{N}^{k}} A^{i, n} x_{t}^{n} \\ 
&=a^{l} x_{t}^{i}+\overline{a}^{l, l} \sum_{j \in \mathcal{N}^{l}, j \neq i} x_{t}^{j}+\sum_{k \in \mathcal{L}, k \neq l} \breve{a}^{l, k} \sum_{n \in \mathcal{N}^{k}} x_{t}^{n} \\
&=a^{l} x_{t}^{i}+\overline{a}^{l, l}\left(\left|\mathcal{N}^{l}\right| \overline{x}_{t}^{l}-x_{t}^{i}\right)+\sum_{k \in \mathcal{L}, k \neq l} \breve{a}^{l, k}\left|\mathcal{N}^{k}\right| \bar{x}_{t}^{k} \\ 
&= A_{l} x_{t}^{i}+\sum_{k \in \mathcal{L}} \left|\mathcal{N}^{k}\right| \breve{a}^{l, k} \bar{x}_{t}^{k}\\
&= A_{l} x_{t}^{i}+\breve{A}_{l} \bar{\mathbf{x}}_{t}.
\end{align*}
Similarly, $B^{i \cdot} \mathbf{u}_{t}=B_{l} u_{t}^{i}+\breve{B}_{l} \bar{\mathbf{u}}_{t}$.
Moreover, we have
\begin{align*}
\mathbf{x}_{t}^{\top} Q \mathbf{x}_{t}&= \sum_{l \in \mathcal{L}} \sum_{k \in \mathcal{L}} \sum_{i \in \mathcal{N}^{l}} \sum_{j \in \mathcal{N}^{k}}\left(x_{t}^{i}\right)^{\top} Q^{i, j} x_{t}^{j} \notag\\
&= \sum_{l \in \mathcal{L}} \sum_{k \in \mathcal{L}, k \neq l} \sum_{i \in \mathcal{N}^{l}} \sum_{j \in \mathcal{N}^{k}}\left(x_{t}^{i}\right)^{\top} \breve{q}^{l, k} x_{t}^{j} \notag\\ 
&\quad+\sum_{l \in \mathcal{L}} \sum_{i \in \mathcal{N}^{l}} \sum_{j \in \mathcal{N}^{l}, j \neq i}\left(x_{t}^{i}\right)^{\top} \breve{q}^{l, l} x_{t}^{j}+\sum_{l \in \mathcal{L}} \sum_{i \in \mathcal{N}^{l}}\left(x_{t}^{i}\right)^{\top} q^{l} x_{t}^{i} \notag\\
&=\sum_{l \in \mathcal{L}} \sum_{k \in \mathcal{L}, k \neq l}\left|\mathcal{N}^{l}\right|\left|\mathcal{N}^{k}\right|\left(\bar{x}_{t}^{l}\right)^{\top} \breve{q}^{l, k} \bar{x}_{t}^{k} +\sum_{l \in \mathcal{L}} \sum_{i \in \mathcal{N}^{l}} \sum_{j \in \mathcal{N}^{l}}\left|\mathcal{N}^{l}\right|^{2}\left(\bar{x}_{t}^{l}\right)^{\top} \breve{q}^{l, l} \bar{x}_{t}^{l}\notag\\
&\quad-\sum_{l \in \mathcal{L}} \sum_{i \in \mathcal{N}^{l}}\left(x_{t}^{i}\right)^{\top} \breve{q}^{l, l} x_{t}^{i}+\sum_{l \in \mathcal{L}} \sum_{i \in \mathcal{N}^{l}}\left(x_{t}^{i}\right)^{\top} q^{l} x_{t}^{i} \notag\\
&=\sum_{l \in \mathcal{L}} \sum_{k \in \mathcal{L}}\left|\mathcal{N}^{l}\right|\left|\mathcal{N}^{k}\right|\left(\bar{x}_{t}^{l}\right)^{\top} \bar{q}^{l, k} \bar{x}_{t}^{k}+\sum_{l \in \mathcal{L}} \sum_{i \in \mathcal{N}^{l}}\left(x_{t}^{i}\right)^{\top}\left(q^{l}-\breve{q}^{l, l}\right) x_{t}^{i} \notag\\
&= \mathbf{x}_{t}^{\top} \breve{Q} \mathbf{x}_{t}+\sum_{l \in \mathcal{L}} \sum_{i \in \mathcal{N}^{l}} \left(x_{t}^{i}\right)^{\top} Q_{l} x_{t}^{i},
\end{align*}
Similarly, 
\[
\mathbf{u}_{t}^{\top} R \mathbf{u}_{t}=\mathbf{u}_{t}^{\top}  \bar{R} \mathbf{u}_{t}+\sum_{l \in \mathcal{L}} \sum_{i \in \mathcal{N}^{l}} \left(u_{t}^{i}\right)^{\top} R_{l} u_{t}^{i}.
\]
Thus, we conclude the proof.

\subsection{Proof of Proposition~\ref{lemma:decpl}}
\label{pf:decpl}

From the coordinate transformation
\[
\tilde{x}_{t}^{i} = {x}_{t}^{i} - \bar{x}_{t}^{l},\quad 
\tilde{u}_{t}^{i} = {u}_{t}^{i} - \bar{u}_{t}^{l},\quad 
\tilde{w}_{t}^{i} = {w}_{t}^{i} - \bar{w}_{t}^{l},
\]
we have that the subpopulation mean fields vanish in \eqref{second_dnmcs}. Therefore, we have
\#\label{eq:decpl_dnms_}
\tilde{x}_{t+1}^{i}&=A_{l} \tilde{x}_{t}^{i}+ B_{l} \tilde{u}_{t}^{i}+\tilde{w}_{t}^{i}.
\#
and
\#\label{eq:second_dnmcs_}
&\tilde x_{t+1}^{i}+ \bar{x}_{t}^{l}=A_{l} \tilde{x}_{t}^{i} +A_{l} \bar{x}_{t}^{l} +B_{l} \tilde{u}_{t}^{i} +B_{l} \bar{u}_{t}^{l} +\bar{A}_{l} \bar{\mathbf{x}}_{t}+\bar{B}_{l} \bar{\mathbf{u}}_{t}+\tilde{w}_{t}^{i}+ \bar{w}_{t}^{l}.
\#
Combining \eqref{eq:decpl_dnms_} and \eqref{eq:second_dnmcs_}, we have
\#
\bar{\mathbf{x}}_{t+1}&=\bar{A} \bar{\mathbf{x}}_{t}+\bar{B} \bar{\mathbf{u}}_{t}+\bar{\mathbf{w}}_{t},
\#
where
\[
\bar{\mathbf{x}}_{t}:=\vec\left(\bar{x}_{t}^{1}, \ldots, \bar{x}_{t}^{L}\right),\quad \bar{\mathbf{u}}_{t}:=\vec\left(\bar{u}_{t}^{1}, \ldots, \bar{u}_{t}^{L}\right),\quad \bar{\mathbf{w}}_{t}:=\vec\left(\bar{w}_{t}^{1}, \ldots, \bar{w}_{t}^{L}\right).
\]
To prove the decomposition shown by \eqref{eq:decpl_cost2}, \eqref{eq:decpl_cost1} and \eqref{eq:decpl_cost}, observe that for any $i \in \mathcal{N}^l$ and $l \in \mathcal{L}$, we have
\[
\sum_{i=1}^{\mathcal{N}^l}\left(x^{i}\right)^{\top} Q_l x^{i}= \sum_{i=1}^{\mathcal{N}^l}\left[\left(\tilde{x}^{i}\right)^{\top} Q_l \tilde{x}^{i}+|\mathcal{N}^l|\cdot (\bar{x}^l)^{\top} Q_l \bar{x}^l\right].
\]
Similar relationship holds for $R^l$ as well. By direct computation, we conclude that 
\#\label{eq:raw_cost}
c_{g t}\left(\mathbf{x}_{t}, \mathbf{u}_{t}\right) = \bar{\mathbf{x}}_{t}^{\top}\bar{Q} \bar{\mathbf{x}}_{t}+\bar{\mathbf{u}}_{t}^{\top}\bar{R}\bar{\mathbf{u}}_{t} +\sum_{l \in \mathcal{L}} \sum_{i \in \mathcal{N}^{l}} \left[\left(\widetilde{x}_{t}^{i}\right)^{\top} Q_{l} \widetilde{x}_{t}^{i}+\left(\widetilde{u}_{t}^{i}\right)^{\top} R_{l} \widetilde{u}_{t}^{i}\right].
\#
After replacing individual states $\{x^i_t\}_{i \in \mathcal{N}^{l}}$ in $\mathbf{x}_{t}$ with $\bar x_t^l$, and similarly for actions, as is shown in \eqref{eq:breve_trans}, we have 
\[
c_{gt}(\breve{\mathbf{x}}_{t}^{l},\breve{\mathbf{u}}_{t}^{l}) =  \bar{\mathbf{x}}_{t}^{\top}\bar{Q} \bar{\mathbf{x}}_{t}+\bar{\mathbf{u}}_{t}^{\top}\bar{R} \bar{\mathbf{u}}_{t} +\sum_{k \in \mathcal{L}, k\neq l}\sum_{i \in \mathcal{N}^{k}} \left[\left(\widetilde{x}_{t}^{i}\right)^{\top} Q_{k} \widetilde{x}_{t}^{i}+\left(\widetilde{u}_{t}^{i}\right)^{\top} R_{k} \widetilde{u}_{t}^{i}\right].
\]
Therefore, we have
\#\label{eq:indi_tilde}
\tilde{c}^l = c_{g t}\left(\mathbf{x}_{t}, \mathbf{u}_{t}\right) - c_{gt}(\breve{\mathbf{x}}_{t}^{l},\breve{\mathbf{u}}_{t}^{l}) = \sum_{i \in \mathcal{N}^{l}} \left[\left(\widetilde{x}_{t}^{i}\right)^{\top} Q_{l} \widetilde{x}_{t}^{i}+\left(\widetilde{u}_{t}^{i}\right)^{\top} R_{l} \widetilde{u}_{t}^{i}\right]
\#
and ${c}^l = {c}^l(\tilde{\mathbf{x}}^l_t,\tilde{\mathbf{u}}^l_t)$ only depends on coordinates in $\mathcal{S}_l$, which shows \eqref{eq:decpl_cost2}. 

After replacing all individual states and actions with the corresponding mean fields, as is shown in \eqref{eq:breve_trans2}, the last term in \eqref{eq:raw_cost} vanishes, and we have
\#\label{eq:indi_bar}
\bar{{c}} = c_{gt}(\breve{\mathbf{x}}_{t},\breve{\mathbf{u}}_{t})
= \bar{\mathbf{x}}_{t}^{\top}\bar{Q}\bar{\mathbf{x}}_{t}
+\bar{\mathbf{u}}_{t}^{\top}\bar{R}\bar{\mathbf{u}}_{t}.
\#
Thus $\bar{c} = \bar{c}\left(\bar{\mathbf{x}}_{t}, \bar{\mathbf{u}}_{t}\right)$ only depends on the mean-fields $\bar{\mathbf{x}}_{t}$ and $\bar{\mathbf{u}}_{t}$. This shows \eqref{eq:decpl_cost1}. Finally, \eqref{eq:decpl_cost} follows directly from \eqref{eq:indi_tilde} and \eqref{eq:indi_bar}.

\subsection{Proof of Proposition~\ref{prop:ergodic_cost_decompositon}}
\label{pf:ergodic_cost_decompositon}
By direct computation, we have
\begin{align*}
C(K)&= \lim _{T \rightarrow \infty}\mathbb{E}_{{\mathbf{x}}_{0} \sim {\mathcal{D}}_{0}, {\mathbf{u}}_{t} \sim \pi_{ K}\left(\cdot | {\mathbf{x}}_{t}\right)}\left[\frac{1}{T} \sum_{t=0}^{T-1}	{c_{gt}}\left({\mathbf{x}}_{t}, {\mathbf{u}}_{t}\right)\right]\notag\\
&= \lim _{T \rightarrow \infty}\mathbb{E}_{{\mathbf{x}}_{0} \sim {\mathcal{D}}_{0}, {\mathbf{u}}_{t} \sim \pi_{ K}\left(\cdot | {\mathbf{x}}_{t}\right)}\left[\frac{1}{T} \sum_{t=0}^{T-1}	{\bar{c}\left(\bar{\mathbf{x}}_{t}, \bar{\mathbf{u}}_{t}\right)+\frac{1}{T}\sum_{t=0}^{T-1}\sum_{l \in \mathcal{L}} \sum_{i \in \mathcal{N}^{l}} \tilde{c}^{l}\left(\tilde{x}_{t}^{i}, \tilde{u}_{t}^{i}\right)}\right]\notag\\
&= \lim _{T \rightarrow \infty} \mathbb{E}_{\bar x_0 \sim \bar{\mathcal{D}}_{0}, \bar u_{t} \sim \pi_{\bar K}\left(\cdot | \bar x_{t}\right)}\left[\frac{1}{T} \sum_{t=0}^{T-1}	\bar{c}\left(\bar{\mathbf{x}}_{t}, \bar{\mathbf{u}}_{t}\right)\right]\notag \\
&\quad+ \sum_{l \in \mathcal{L}} \sum_{i \in \mathcal{N}^{l}} \mathbb{E}_{\tilde x_0^i \sim \tilde{\mathcal{D}}_{0}^l, \tilde u_{t}^i \sim \pi_{ K_l}\left(\cdot | \tilde x_{t}^i\right)} \left[\frac{1}{T} \sum_{t=0}^{T-1} \tilde{c}^{l}\left(\tilde{x}_{t}^{i}, \tilde{u}_{t}^{i}\right)\right].\notag\\
&=: \bar C(\bar K) + \sum_{l \in \mathcal{L}}   \tilde C( K_l).
\end{align*}

\subsection{Proof of Lemma~\ref{lemma:dom}}

Lemma \ref{lemma:cost_grad_K_single} shows that 
\begin{equation*}
C(K)=\tr\left(P_{K} \Phi_{\sigma}\right)+\sigma^{2} \cdot \tr(R)=\mathbb{E}_{x \sim N\left(0, \Phi_{\sigma}\right)}\left(x^{\top} P_{K} x\right)+\sigma^{2} \cdot \tr(R).
\end{equation*} 
Given two linear policies $K$ and $K^{\prime}$, we define the function
\begin{align}
A_{K, K^{\prime}}(x) \coloneqq 2 x^{\top}\left(K^{\prime}-K\right)^{\top} E_{K} x+x^{\top}\left(K^{\prime}-K\right)^{\top}\left(R+B^{\top} P_{K} B\right)\left(K^{\prime}-K\right) x\notag.
\end{align}
Let $x^*$ and $u^*$ denote the states and actions induced by $K^*$, respectively. Lemma~\ref{lemma:cost_diff} then gives us 
\begin{align}\label{eq:quad_lower_bound}
C(K)-C(K^*) 
& = -\mathbb{E}_{x \sim N\left(0, \Phi_{\sigma}\right)}  \sum_{t \geq 0} A_{K, K^*}\left(x_{t}^*\right) \notag\\
& \leq \tr\left[\Sigma_{K^*}  E_{K}^{\top}\left(R+B^{\top} P_{K} B\right)^{-1} E_{K}\right] .
\end{align}
Since $R+B^{\top} P_{K} B \succeq R$, we have
\begin{equation*}
\begin{aligned}
\tr\left[\Sigma_{K^*}  E_{K}^{\top}\left(R+B^{\top} P_{K} B\right)^{-1} E_{K}\right]
&\leq\left\|\Sigma_{K^{*}}\right\| \cdot\|(R+B^{\top} P_{K} B)^{-1}\| \cdot \tr\left(E_{K}^{\top} E_{K}\right)\\
&\leq\frac{\left\|\Sigma_{K^{*}}\right\|}{\sigma_{\min }(R)} \tr\left(E_{K}^{\top} E_{K}\right),
\end{aligned}
\end{equation*}
which completes the upper bound proof.

Next, we establish a lower bound. Since the policy $K^{\prime}:=K-\left(R+B^{\top} P_{K} B\right)^{-1} E_{K}$ attains the equality in \eqref{eq:quad_lower_bound}, we have
\begin{align*}
C(K)-C\left(K^{*}\right)
&\geq C(K)-C\left(K^\prime\right)\notag\\
&= -\mathbb{E}_{x \sim N\left(0, \Psi_{\sigma}\right)} \sum_{t \geq 0} A_{K, K^{\prime}}\left(x_{t}^{\prime}\right)\notag\\
&= \tr\left[\Sigma_{K^{\prime}} E_{K}^{\top}\left(R+B^{\top} P_{K} B\right)^{-1} E_{K}\right]\notag\\
&\geq \sigma_{\min }(\Phi) \cdot\left\|R+B^{\top} P_{K} B\right\|^{-1} \cdot \tr\left(E_{K}^{\top} E_{K}\right).
\end{align*}
This completes the proof.

\subsection{Proof of Lemma~\ref{lemma:exactgradient}}
We first make the induction assumption $C(K_n)\leq C(K_0)$. Applying Lemma~\ref{lemma:cost_diff} to policies $K_{n}$ and $K^\prime_{n+1} =K_{n}-\eta \cdot E_{K_n}$, we have 
\#\label{eq:cost_diff_bound}
C&\left(K_{n+1}^{\prime}\right)-C\left(K_{n}\right) \notag\\
&= -2 \eta \cdot \tr\left(\Sigma_{K_{n+1}^{\prime}} \cdot E_{K_{n}}^{\top} E_{K_{n}}\right)+\eta^{2} \cdot \tr\left[\Sigma_{K_{n+1}^{\prime}} \cdot E_{K_{n}}^{\top}\left(R+B^{\top} P_{K_{n}} B\right) E_{K_{n}}\right]\notag\\
&\leq -2 \eta \cdot \tr\left(\Sigma_{K_{n+1}^{\prime}} \cdot E_{K_{n}}^{\top} E_{K_{n}}\right)+\eta^{2} \cdot\left\|R+B^{\top} P_{K_{n}} B\right\| \cdot \tr\left(\Sigma_{K_{n+1}^{\prime}} \cdot E_{K_{n}}^{\top} E_{K_{n}}\right).
\#
Note that we also have
\#\label{eq:Unif_bound_K0}
\tr\left[\Sigma_{K_{n+1}^{\prime}} \cdot E_{K_{n}}^{\top}\left(R+B^{\top} P_{K_{n}} B\right) E_{K_{n}}\right] \leq \left\|R+B^{\top} P_{K_{n}} B\right\| \cdot \tr\left(\Sigma_{K_{n+1}^{\prime}} \cdot E_{K_{n}}^{\top} E_{K_{n}}\right).
\#
Furthermore by Lemma~\ref{bound_mats} and the induction assumption $C(K_n)\leq C(K_0)$, we have
\[
\left\|R+B^{\top} P_{K_{n}} B\right\| \leq \|R\|+\|B\|^{2} \cdot\left\|P_{K_{n}}\right\| \leq \|R\|+\sigma_{\min }^{-1}(\Phi) \cdot\|B\|^{2} \cdot C\left(K_{0}\right).
\]
Since the step size $\eta$ satisfies $\eta\leq \left[\|R\|+\sigma_{\min }^{-1}(\Phi) \cdot\|B\|^{2} \cdot C\left(K_{0}\right)\right]^{-1}$, combining \eqref{eq:cost_diff_bound} and \eqref{eq:Unif_bound_K0}, we conclude
\begin{equation}\label{eq:like_the_dom}
C\left(K_{n+1}^{\prime}\right)-C\left(K_{n}\right) \leq - \eta \cdot \tr\left(\Sigma_{K_{n+1}^{\prime}} \cdot E_{K_{n}}^{\top} E_{K_{n}}\right) 
\leq - \eta \cdot\sigma_{\min }(\Phi) \cdot\tr\left( E_{K_{n}}^{\top} E_{K_{n}}\right) ,
\end{equation}
where we use the fact that $\Sigma_{K_{n+1}^{\prime}} \succeq \Phi$. Combining \eqref{eq:like_the_dom} and \eqref{eq:dom_l} in Lemma~\ref{lemma:dom}, we conclude \eqref{eq:exactcostreduction}. Then \eqref{eq:linear_in_exact} follows directly by adding $C\left(K_{n}\right)-C\left(K^{*}\right)$ to both sides of \eqref{eq:exactcostreduction}. Thus, we conclude the proof.

\subsection{Proof of Lemma~\ref{lemma:bound_update_diff}}

We will use Lemma~\ref{lemma:perturbation} in the proof. We first show that its condition \eqref{eq:policyboundcondition} holds, which is equivalent to
\begin{align}\label{eq:pertur_condi}
4  \left(1+\left\|A-B K_{n+1}^{\prime}\right\|\right) \cdot \|B\| \cdot \|\Sigma_{K_{n+1}^{\prime}}\|\cdot \left\|K_{n+1}-K_{n+1}^{\prime}\right\|  \leq \sigma_{\min }(\Phi).
\end{align}
By direct computation, we have
\begin{equation}\label{eq:ABKbound}
\left\|A-B K_{n+1}^{\prime}\right\| \leq \left\|A-B K_{n}\right\|+\eta\cdot\|B\|\cdot\|E_{K_n}\|.
\end{equation}
In the following, we bound $\|E_{K_n}\|$ and $\|\Sigma_{K_{n+1}^{\prime} }\|$ with problem parameters. For $\|E_{K_n}\|$, we have
\#\label{Ebound}
\|E_{K_n}\|&=\| (R+B^\top P_{K_n} B)K_n - B^\top P_{K_n}A \| \notag\\
&\leq \|R+B^\top P_{K_n} B \|\cdot \|K_n\| + \|B^\top\|\cdot\| P_{K_n}\|\cdot\|A\| \notag\\
&\leq \left( \|R\|+\| B\|^2\cdot\frac{C(K_0)}{\sigma_{min}(\Phi)} \right)\cdot\|K_n\|+\|B\|\cdot\|A\|\cdot\frac{C(K_0)}{\sigma_{min}(\Phi)},
\#
where in the last inequality we have used Lemma \ref{bound_mats} and the induction assumption. For $\|\Sigma_{K_{n+1}^{\prime} }\|$, again by Lemma \ref{bound_mats} and the induction assumption, we have 
\begin{align}\label{sigmabound}
\|\Sigma_{K_{n+1}^{\prime}}\| \leq \frac{C(K^\prime_{n+1})}{\sigma_{min}(Q)} \leq \frac{C(K_0)}{\sigma_{min}(Q)}.
\end{align}
Furthermore, using \eqref{eq:single_assump3}, we can bound $\left\|K_{n+1}-K_{n+1}^{\prime}\right\|$ as
\#\label{Kdiffbound}
\left\|K_{n+1}-K_{n+1}^{\prime}\right\| = \eta\left\|\widehat{E}_{K_n}-E_{K_n}\right\| \leq \eta \cdot \delta_n.
\#
Combining \eqref{eq:ABKbound}, \eqref{Ebound}, \eqref{sigmabound} and \eqref{Kdiffbound}, we conclude that
\#
4  \left(1+\left\|A-B K_{n+1}^{\prime}\right\|\right) \cdot \|B\| \cdot \|\Sigma_{K_{n+1}^{\prime}}\|\cdot \left\|K_{n+1}-K_{n+1}^{\prime}\right\|  \leq \Lambda_1(\|K_n\|,C(K_0))\cdot\eta \cdot\delta_n,
\#
where $\Lambda_1(\cdot,\cdot)$ is a polynomial of $\|K_n\|$ and $C(K_0)$. Here we regard $A$, $B$, $Q$, $R$ and  $\Phi$  as fixed parameters. Suppose $\delta_n$ satisfies
\begin{equation}\label{eq:delta_cond1}
0 < \delta_n \leq \sigma_{\min }(\Phi)\cdot 1/\Lambda_1(\|K_n\|,C(K_0)),
\end{equation} 
and, therefore, the condition~\eqref{eq:pertur_condi} holds. Then by Lemma~\ref{lemma:perturbation}, we have 
\begin{multline}
\label{eq:costperturbationbound2}
\left\|C({K_{n+1}})-C({K_{n+1}^{\prime}})\right\| 
\leq 6 \left(\|\Phi\| +\sigma^{2} \cdot\|B\|^{2}\right)\cdot \sigma_{\min }^{-1}(\Phi) \cdot\left\|\Sigma_{K_{n+1}^{\prime}}\right\| \cdot\|K_{n+1}^{\prime}\| \cdot\|R\| \\
\cdot(\|K_{n+1}^{\prime}\| \cdot\|B\| \cdot\|A-B K_{n+1}^{\prime}\|+\|K_{n+1}^{\prime}\| \cdot\|B\|+1) \cdot\left\|K_{n+1}-K_{n+1}^{\prime}\right\|.
\end{multline}
Note that we can further bound $\|K_{n+1}^{\prime}\|$ by
\#\label{exactKupdatebound}
\|K_{n+1}^{\prime}\| = \|K_{n}-\eta \cdot E_{K_n}\|
\leq\|K_{n}\| + \eta\cdot\| E_{K_n}\|.
\#
Combining \eqref{eq:ABKbound}, \eqref{Ebound}, \eqref{Kdiffbound}, \eqref{eq:costperturbationbound2} and \eqref{exactKupdatebound}, we conclude that 
\#
\left|C\left(K_{n+1}\right)-C\left(K_{n+1}^{\prime}\right)\right| 
\leq \Lambda_2(\|K_n\|,C(K_0))\cdot\eta\cdot\delta_n,
\#
where $\Lambda_2(\|K_n\|,C(K_0))$ is a polynomial of $\|K_n\|$ and $C(K_0)$. Again,
here we regard $A$, $B$, $Q$, $R$,  $\Phi$ and $\sigma^2$ as fixed parameters.
Suppose  we have 
\#\label{eq:delta_cond2}
0<\delta_n\leq \epsilon / 2 \cdot \sigma_{\min }(\Phi) \cdot \sigma_{\min }(R) \cdot\left\|\Sigma_{K^{*}}\right\|^{-1}\cdot 1/\Lambda_2(\|K_n\|,C(K_0)),
\#
and, hence,
$
\left|C\left(K_{n+1}\right)-C\left(K_{n+1}^{\prime}\right)\right| \leq \epsilon / 2\cdot\eta \cdot \sigma_{\min }(\Phi) \cdot \sigma_{\min }(R) \cdot\left\|\Sigma_{K^{*}}\right\|^{-1}.
$
Finally, without loss of generality we assume $\epsilon \leq 1$ and specify $\Lambda(\cdot,\cdot)$ in \eqref{eq:single_assump3_} to be
\#
\Lambda(\|K_n\|,C(K_0)) = \sigma_{\min }(R) \cdot\left\|\Sigma_{K^{*}}\right\|^{-1}\cdot\Lambda_1(\|K_n\|,C(K_0)) +  \Lambda_2(\|K_n\|,C(K_0)).
\#
This ensures  \eqref{eq:delta_cond1} and \eqref{eq:delta_cond2}, hence  concluding the proof.

\section{ Numerical Experiments}\label{sec:experiments}

We illustrate the empirical performance of the hierarchical natural actor-critic algorithm. We consider simulated LQR settings with varied numbers of agents in subpopulations to illustrate the computational complexity of our algorithm. 

We consider an LQR problem with two exchangeable subpopulations. Concretely, following the notation in \S\ref{section:main_result} (also see the notation table in \S\ref{sec:notation_table}),  we sample the system matrices of the two subpopulations as $A_l \sim \mathcal{N} (a_l  \mathrm{1}_{d_l}, \sigma^2I_{d_l})$, $B_l \sim \mathcal{N} (b_l  \mathrm{1}_{d_l}, \sigma^2I_{d_l})$, $Q_l \sim \mathcal{N} (q_l  \mathrm{1}_{d_l}, \sigma^2I_{d_l})$, $R_l \sim \mathcal{N} (r_l  \mathrm{1}_{d_l}, \sigma^2I_{d_l})$, $l\in\{1,2\}$. We set all off-diagonal terms to be the same, namely $\breve{a}^{l, k} \equiv \breve{a}, \breve{b}^{l, k}\equiv \breve{b}, \breve{q}^{l, k} \equiv \breve{q}, \breve{r}^{l, k} \equiv \breve{r}$ for all $k,l \in \mathcal{L}$. Then the original multi-agent system matrices $A, B, Q, R$ are recovered so that they satisfy partial exchangeability. 

We measure the computational complexity via the number of iterations needed for the algorithm to achieve a predetermined precision $\epsilon$. During our experiment, we let each subpopulation have the same size and vary it from $50$ to $500$. We fix  $d_1 = d_2 = 2$, $\sigma^2 = 0.01$, and sample $a_1, b_1, q_1, r_1 \sim \text{Unif}[0.06,0.08]$, $a_2, b_2, q_2, r_2 \sim \text{Unif}[0.12,0.14]$ and $\breve{a}, \breve{b},\breve{q},\breve{r} \sim 1.6 \text{e-4}\text{Unif}[0,1]$. We average simulation results over 20 independent runs and plot the number of iterations needed to achieve the precision  $\epsilon = 1\times10^{-5}$ against the subpopulation size. The results are shown in Figure \ref{fig:complexity}. We see that as the subpopulation size increases, the computational complexity remains roughly fixed, rather than increasing with the size of the subpopulation. This coincided with our theoretical analysis which showed that the computational complexity of Algorithm~\ref{algo:multi-agent_LQR} is independent of the number of agents in subpopulations.  
 
\begin{figure}
	\includegraphics[width=.9\textwidth]{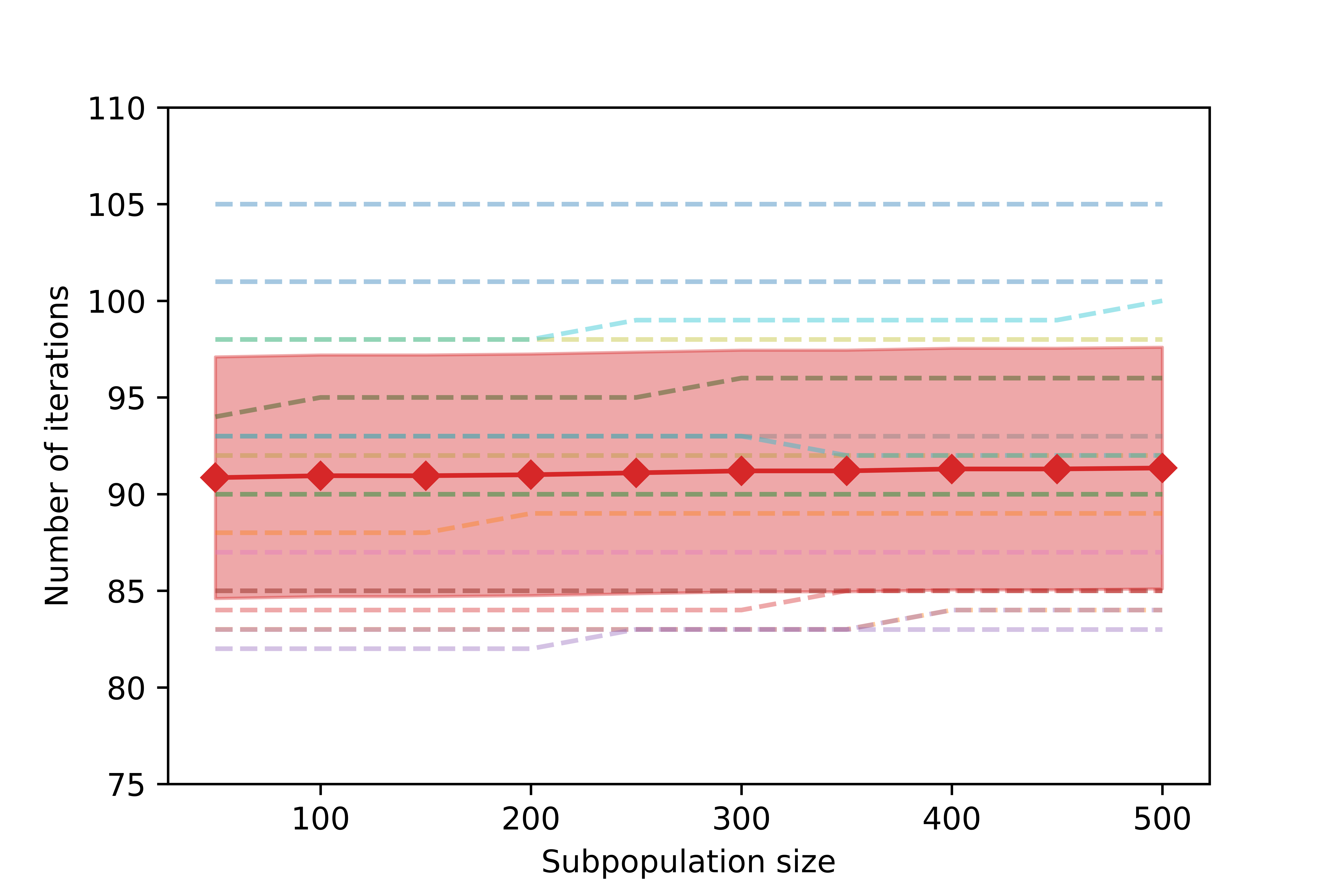}
	\caption{The number of iterations needed for the hierarchical natural actor-critic algorithm to converge to a predetermined precision as a function of the subpopulation size. Each dashed line represents the number of iterations of one simulated LQR instance.  The red solid line represent the average number of iterations and the red shaded region indicates $1$-standard deviation. As the subpopulation size increases, the number of iterations needed are roughly fixed for all instances, validating that the computational complexity of the hierarchical natural actor-critic algorithm is independent of the number of agents in each subpopulation.}\label{fig:complexity}
\end{figure}

\acks{This work was completed in part with resources provided by the
	University of Chicago Research Computing Center.}

\newpage

\appendix

\section{Policy Evaluation Algorithm}
\label{section:gtd}

For completeness, we introduce a policy evaluation algorithm based on gradient-based temporal difference learning (Algorithm 2 in \cite{yang2019global}) that can be implemented in Algorithm \ref{algo:multi-agent_LQR}. 

In the LQR setting, the state- and action-value functions of policy $\pi_K$ take, respectively, the following forms
\begin{equation}\label{eq:value_funrion}
V_{K}(x) =\sum_{t=0}^{\infty}\left\{\mathbb{E}\left[c\left(x_{t}, u_{t}\right)|x_0 = x \right]-C(K)\right\}, \quad Q_K(x,u)=c(x, u)-C(K)+\mathbb{E}_{x^\prime}\left[V_{K}\left(x^{\prime}\right) \right],
\end{equation}
where $\{(x_t,u_t)\}_{t\geq0}$ is the sequence of state-action pairs generated by policy $\pi_K$ with initial state $x_0 = x$. Here we use $x^\prime$ to denote the next state generated by $(x,u)$, i.e. $x^\prime = Ax+Bu+w$, where the noise $w\sim N\left(0, \Phi\right)$ for some positive definite matrix $\Phi$. 

The following lemma establishes the functional forms of the state- and action-value functions.

\begin{lemma}\label{value_func}
	The state and action value functions \eqref{eq:value_funrion} are given, respectively, by
	\#
	V_{K}(x)&=x^{\top} P_{K} x-\tr\left(P_{K} \Sigma_{K}\right),\\
	Q_{K}(x, u)&=\varphi(v)^\top\delta_{K}^{*}-\sigma^{2} \cdot \tr(R+P_{K} B B^{\top})-\tr\left(P_{K} \Sigma_{K}\right),
	\#
	where $v = \vec(x, u)$ and $\varphi(v)=\svec(v v^\top)$. The vector $\delta_{K}^{*}$ is called \emph{value vector} and is related to matrices $A$, $B$, $Q$, $R$ and matrix $P_K$ defined in Lemma~\ref{lemma:cost_grad_K_single}. Moreover, $E_K$ can be recovered with $\delta_K^*$ and $K$. The explicit form of $\delta_{K}^{*}$	is given in the proof in \S\ref{pf:value_func}.
\end{lemma}

By Lemma \ref{value_func}, estimating the action-value function $Q_{K}(x, u)$ is equivalent to estimating  $\delta_K^*$ and thus gives estimator of $E_K$ as we desire.

Lemma \ref{lemma:chv} tells us that $v = [x^\top, u^\top]^\top \sim N (0, \breve{\Sigma}_{K})$ is a Markov chain. Hereafter, we use $c(v)$ to denote $c(x, u)$. Similar notation is also used for other functions.

Note that $V_{K}(x)=\mathbb{E}_{u \sim \pi_{K}(\cdot | x)}\left[Q_{K}(v)\right]$,  which indicates that for any $v \in \mathbb{R}^{d+k}$, we have
\begin{equation}
\left\langle\varphi(v), \delta_{K}^{*}\right\rangle= c(v)-C(K)+\left\langle\mathbb{E}\left[\varphi\left(v^{\prime}\right)\right], \delta_{K}^{*}\right\rangle,
\end{equation}
where $v^\prime =\vec \left(x^{\prime}, u^{\prime}\right)$ is the next state-action pair  generated by $v = \vec(x,u)$. Denote the expectation with respect to $v\sim N (0, \breve{\Sigma}_{K})$ by $\mathbb{E}_{v}$. We let
\[
b_k=(C(K) \  d_K)^\top, \ 
\Theta_{K}=\mathbb{E}_{v}\left\{\varphi(v)\left[\varphi(v)-\varphi\left(v^{\prime}\right)\right]^{\top}\right\},
\text{ and }
d_{K}=\mathbb{E}_{v}[c(v) \varphi(v)].
\]
We also let $\gamma_{K}^{*}=\left(C(K), \delta_{K}^{* \top}\right)^{\top}$. It was shown in \cite{yang2019global} that $\left(\gamma_{K}^{*},0\right)^{\top}$ is the unique saddle point of the minimax optimization problem:
\begin{equation}\label{eq:new_opt}
\begin{aligned} \min _{\gamma \in \mathcal{X}_{\Gamma}} \max _{\xi \in \mathcal{X}_{\Xi}} G(\gamma, \xi)=\left[\gamma^{1}-C(K)\right] \cdot \xi^{1} +\left\langle\gamma^{1} \cdot \mathbb{E}_{v}[\varphi(v)]+\Theta_{K} \gamma^{2}-d_{K}, \xi^{2}\right\rangle - 1 / 2 \cdot\|\xi\|_{2}^{2},
\end{aligned}
\end{equation}
where $\mathcal{X}_{\Gamma}$ and $\mathcal{X}_{\Xi}$ are compact sets and $\gamma=(\gamma^1,\gamma^2)$ and $\xi = (\xi^1,\xi^2)$ are primal and dual variables.

We solve \eqref{eq:new_opt} with stochastic gradient method, and return $\hat\delta_{K} = \hat\gamma^{2}$ as the estimator of $\delta_{K}^{*}$. Algorithm~\ref{algo:gtd} details the Gradient-Based Temporal Difference (GTD) Algorithm.

To present the theoretical result of the policy evaluation algorithm using GTD, we make the following assumptions. First, we specify the compact sets $\mathcal{X}_{\Gamma}$ and $\mathcal{X}_{\Xi}$ defined in \eqref{eq:new_opt} to be
\begin{equation}
\begin{aligned}
\label{projsets}
\mathcal{X}_{\Gamma} &= \{\gamma:0\leq\gamma^1\leq \Gamma^1, \| \gamma^2 \|_2\leq  \Gamma^2 \},\\
\mathcal{X}_{\Xi} &= \{\xi:0\leq|\xi^1|\leq \Xi^1= C(K_0), \| \xi^2 \|_2\leq  \Xi^2 \},
\end{aligned}
\end{equation}
where  
$\Gamma^1 = C(K_0)$, 
$\Gamma^2 = \| Q \| _{{\mathrm{F}}}+ \| R \|_{{\mathrm{F}}} + \sqrt{ d }  / \sigma_{\min} (\Phi)  \cdot   ( \| A \|_{{\mathrm{F}}} ^2 + \| B \|_{{\mathrm{F}}}^2   )\cdot C(K_0 )$, 
$\Xi^1 =  C(K_0)$, and 
$\Xi^2 =  C\cdot \left(1+\|K\|_{\mathrm{F}}^{2}\right)^{2} \cdot   \Gamma^2 \cdot \sigma_{\min}^{-2} (Q) \cdot [C(K_0)]^2$.
Here, $C$ is a constant that does not depend on $K$. 
Moreover, we set the step size in Algorithm \ref{algo:gtd} to be $\alpha_t = \alpha/\sqrt{t}$.

\begin{algorithm} [t]
	\caption{Gradient-Based Temporal-Difference (Algorithm 2 in \cite{yang2019global})} 
	\label{algo:gtd} 
	\begin{algorithmic} 
		\STATE{{\textbf{Input:}} Current policy $\pi_K$, number of iterations $T$, and  step sizes  $\{ \alpha_t=\alpha/\sqrt{t}\}_{t\in\{1,2, \dots, T\} } $.}  
		\STATE{\textbf{Initialization:} Initialize  $\gamma_0 \in \mathcal{X}_{\Gamma} $ 
			and $\xi_0\in \mathcal{X}_{\Xi} $. Sample the initial state-action pairs 
			$x_0\sim \mathcal{D}_K$, $u_0\sim \pi_{K}\left(\cdot | x_{0}\right)$, and construct $v_0$.
			Observe the cost $c_0$ and the next state $x_1$.}
		\FOR{$t= 1 ,  2, \ldots,  T $}
		\STATE{Take action $u_{t}\sim \pi_{K}\left(\cdot | x_{t}\right)$ , set $v_{t} = \vec (x_t, u_t)$, observe the cost $c_t$ and the next state $x_{t+1}$.}
		\STATE{Update primal-dual parameters $\gamma_{t-1} $ and $\xi_{t-1}$ with gradient decent to obtain $\gamma_t $ and $\xi_t$.}
		\STATE{Project $\gamma_t $ and $\xi_t$ to $\mathcal{X}_{\Gamma}$ and $\mathcal{X}_{\Xi}$.} 
		\ENDFOR
		\STATE{Obtain $\hat \gamma =(\hat \gamma^1, \hat \gamma^2)  =   (  \sum_{t=1}^{T} \alpha_t \cdot \gamma_t )/ (  \sum_{t=1}^{T} \alpha _t ) $ and $\hat \xi =  (  \sum_{t=1}^{T} \alpha_t \cdot \xi_t ) / (  \sum_{t=1}^{T} \alpha _t ) $}.
		\STATE{\textbf{Output:} Estimators $\hat C = \hat \gamma^1$ of cost $C(K)$ and  $\hat \delta_K =\hat \gamma^2$ of $\delta_K^*$.}
	\end{algorithmic}
\end{algorithm}

\begin{theorem}[Policy evaluation algorithm using GTD]\label{pe}
	Let initial policy $\pi_{K_0}$ be stable in the sense that $\rho(A-BK_0)<1$. Let $\mathcal{X}_{\Gamma}$ and $\mathcal{X}_{\Xi}$ be compact sets specified in \eqref{projsets}. For any $\rho \in(\rho(A-B K), 1)$, there exists sufficiently large iteration number $T$, so that with probability at least $1-T^{-4}$, we have 
	\begin{align}
	\| \hat\delta_{K} -\delta_{K}^{*} \|_{2}^2 \leq \Lambda\left(\Gamma^2,\Xi^2,C(K_0),\|K\|_{\mathrm{F}}, \Sigma_K, \sigma_{\min }^{-1}(Q)\right)\cdot\frac{{\log^6 T}}{(1-\rho){\tau_{K}^{*}}^{2} \sqrt{T}},
	\end{align}
	where $\Lambda$ is a polynomial of 
	$\Gamma^2$, 
	$\Xi^2$,
	$C(K_0)$, 
	$\|K\|_{\mathrm{F}}$, 
	$\Sigma_K$ and 
	$\sigma_{\min }^{-1}(Q)$, and constant 
	${\tau_{K}^{*}}>0$ solely depends on $\rho(A-B K)$, $\sigma$, and $\sigma_{\min }(\Phi)$.
\end{theorem}

\subsection{Proof Sketch of Theorem \ref{pe}}\label{pf:pe}

We will use Lemma~\ref{lemma:saddle_gap} to establish the result. We start by
verifying the conditions of Lemma~\ref{lemma:saddle_gap}. The following lemma shows that 
$(\gamma_{K}^{*},0)$ is the solution to the optimization problem \eqref{eq:new_opt}. 

\begin{lemma}\label{lemma:compact_sets}
	It holds that $\gamma_{K}^{*} \in \mathcal{X}_{\Gamma}$. Let $\xi(\gamma)$ be the solution to the unconstrained problem $\max _{\xi} G(\gamma, \xi)$. Then for any $\gamma \in \mathcal{X}_{\Gamma}$, we have $\xi(\gamma) \in \mathcal{X}_{\Xi}$.
\end{lemma}
\begin{proof}
	See \S\ref{pf:compact_sets} for a detailed proof.
\end{proof}

For fixed values $\hat\gamma\in\Gamma$ and $\hat\xi\in\Xi$, we let 
\[
G_{\hat\gamma\cdot} = \max _{\xi\in\Xi}G(\hat\gamma,\xi), \quad G_{\cdot\hat\xi} = \min _{\gamma\in\Gamma}G(\gamma,\hat\xi).
\]
The primal-dual gap with respect to $\hat\gamma$ and $\hat\xi$ is defined as $G_{\hat\gamma\cdot} - G_{\cdot\hat\xi}$, which measures the closeness between $(\hat\gamma,\hat\xi)$ and the saddle point $(\delta_K^*,0)$.
The rate of convergence of the estimators obtained by the primal-dual problem \eqref{eq:new_opt}  is controlled by the primal-dual gap $G_{\hat\gamma\cdot} - G_{\cdot\hat\xi}$ as shown by the following lemma.

\begin{lemma}\label{lemma:esitimate_bound}
	It holds that
	\begin{equation}\label{gaplink}
	{\tau_{K}^{*}}^2 \left( |\gamma^1-C(K)|^2 + \|\hat\delta_{K} - \delta_{K}^{*}\|^2_2\right)\leq G_{\hat\gamma\cdot} - G_{\cdot\hat\xi} ,
	\end{equation}
	where $\tau_{K}^{*}>0$ is a constant that depends only on $\rho(A-B K)$, $\sigma$, and $\sigma_{\min }(\Phi)$.
\end{lemma}
\begin{proof}
	See \S\ref{pf:esitimate_bound} for a detailed proof.
\end{proof}

Next, we construct an upper bound on $G_{\hat\gamma\cdot} - G_{\cdot\hat\xi}$. For technical reasons, it is convenient to consider the case where $\|v_t\|^2=\left\|x_{t}\right\|_{2}^{2}+\left\|u_{t}\right\|_{2}^{2}$ is bounded by some value that depends on $T$, for $0 \leq t \leq T$. This will guarantee better Lipschitz properties of $G(\gamma, \xi)$ and thus enable us to bound the primal-dual gap. We use the following lemma to characterize the tail distribution of $v$. 

\begin{lemma}\label{lemma:tuncation_prob}
	Consider the event 
	\[
	\mathcal{A} = \left\{\|v\|_{2}^{2}-\tr\left(\Breve{\Sigma}_{K}\right) \leq C^\prime \cdot \log T \cdot\|\Breve{\Sigma}_{K}\|\right\},
	\]
	where $v \sim N (0, \breve{\Sigma}_{K})$ and $\Breve{\Sigma}_{K}$ is defined in Lemma~\ref{lemma:chv}. 
	We have $\mathbb{P}(\mathcal{A}^c)\leq T^{-6}$ when $C^\prime$ is large enough.
\end{lemma}
\begin{proof}
	See \S\ref{pf:tuncation_prob} for a detailed proof.
\end{proof}

Hereafter, we consider the optimization problem in \eqref{eq:new_opt} conditioned on event $\mathcal{A}$ defined in Lemma~\ref{lemma:tuncation_prob}. Define the truncated feature vector by $\tilde{\varphi}(v)=\varphi(v)\cdot\ind_{\mathcal{A}}$. Let $\tilde \Theta_K$ and $\tilde d_K$ be defined as $\Theta_K$ and $d_K$, but with $\varphi(v)$ and $\varphi(v^\prime)$  replaced with $\tilde{\varphi}(v)$ and $\tilde{\varphi}(v^\prime)$, respectively. Consider the new optimization problem 
\begin{equation}\label{eq:trunc_new_opt}
\begin{aligned} \min _{\gamma \in \mathcal{X}_{\Gamma}} \max _{\xi \in \mathcal{X}_{\Xi}} \tilde G(\gamma, \xi)=\left[\gamma^{1}-C(K)\right] \cdot \xi^{1} +\left\langle\gamma^{1} \cdot  \mathbb{E}_{v}[\tilde\varphi(v)]+\tilde\Theta_{K} \gamma^{2}-\tilde d_{K}, \xi^{2}\right\rangle - 1 / 2 \cdot\|\xi\|_{2}^{2}.
\end{aligned}
\end{equation}
For sufficiently large $T$, the objectives in \eqref{eq:new_opt} and \eqref{eq:trunc_new_opt}
are close.
\begin{lemma}\label{lemma:bound_obj_diff_tunc}
	When $T$ is sufficiently large, it holds that 
	\begin{equation}
	|G(\gamma, \xi)-\widetilde{G}(\gamma, \xi)| \leq \frac{1}{T}
	\end{equation}
	for any $\gamma \in \mathcal{X}_{\Gamma}$, $\xi \in \mathcal{X}_{\Xi}$.
\end{lemma}
\begin{proof}
	See \S\ref{pf:bound_obj_diff_tunc} for a detailed proof.
\end{proof}

A direct corollary is that  the difference of primal-dual gaps between optimization problems \eqref{eq:new_opt} and \eqref{eq:trunc_new_opt} can be bounded as 
\begin{equation}\label{eq:gap_bound}
|(G_{\hat\gamma\cdot} - G_{\cdot\hat\xi}) -  (\tilde G_{\hat\gamma\cdot} - \tilde G_{\cdot\hat\xi})|\leq \frac{2}{T}.
\end{equation}

The objective function of the truncated optimization problem \eqref{eq:trunc_new_opt} has good properties as characterized by the following lemma.  

\begin{lemma}\label{lemma:lip}
	As a function of $\gamma$ and $\xi$, the norm of $\nabla \tilde G(\gamma,\xi)$  can be bounded by 
	\begin{equation}
	3 (\Gamma^2+\Xi^2)\cdot(C^\prime \cdot \log T+d+k)^2[\sigma^2+(1+\|K\|_F^2)\|\Sigma_K\|]^2 =: L_1,
	\end{equation}
	for a sufficiently large  $C^\prime$. Thus, $\widetilde{G}(\gamma, \xi)$ is Lipschitz for both $\gamma$ and $\xi$ with finite a constant $L_1$. Moreover, we have $\nabla^2_{\gamma\gamma} \tilde G(\gamma,\xi) = 0$ and $\nabla^2_{\xi\xi} \tilde G(\gamma,\xi) = -I$, where $I$ is the identity matrix of proper dimension. 		
\end{lemma}
\begin{proof}
	See \S\ref{pf:lip} for a detailed proof.
\end{proof}

Finally, Lemma \ref{lemma:chv} shows that $\{v_t\}_{t\geq0}$ is a geometrically $\beta$-mixing stochastic process with a parameter $\rho \in (\rho(A-BK),1)$ and thus mixes rapidly. Therefore, we have verified the conditions of Lemma~\ref{lemma:saddle_gap}, which gives an upper bound on the primal-dual gap in \eqref{eq:trunc_new_opt}. We have the following result.

\begin{lemma}\label{lemma:primaldual}
	For a sufficiently large $T$, we have
	\begin{align}\label{eq:gapbound2}
	\tilde G_{\hat\gamma\cdot} - \tilde G_{\cdot\hat\xi}
	\leq \Lambda_3\frac{{\log^2 T}+\log(T^5)}{ \log (1 / \rho)\cdot\sqrt{T}} + \Lambda_4 \frac{\log^6 T}{(1-\rho)T},
	\end{align}
	with probability at least $1-T^{-5}$, where $\Lambda_3$ and $\Lambda_4$ are polynomials of $\Gamma^2$, $\Xi^2$, $C(K_0)$, $\|K\|_{\mathrm{F}}$, $\Sigma_K$, and $\sigma_{\min }^{-1}(Q)$.
\end{lemma}
\begin{proof}
	See \S\ref{pf:primaldual} for a detailed proof.
\end{proof}

Recall that for the event $\mathcal{A}_t=\left\{\|v_t\|_{2}^{2}-\tr\left(\breve{\Sigma}_{K}\right)>C^\prime \cdot \log T \cdot\left\|\breve{\Sigma}_{K}\right\|\right\}$, we have $P(\mathcal{A}) \leq T^{-6}$ for $C_1$ large enough. Then the event $\bigcap_{t=1}^{T}\mathcal{A}_t$ holds with probability at least $1-T^{-5}$. Combining \eqref{eq:gapbound2} and \eqref{eq:gap_bound}, it holds with probability at least $1-2T^{-5}>1-T^{-4}$ that the primal-dual gap of the original optimization problem \ref{eq:new_opt} can be bounded as
\begin{align}\label{eq:gapbound3}
G_{\hat\gamma\cdot} -  G_{\cdot\hat\xi}
\leq \Lambda_3\frac{{\log^2 T}+\log(T^5)}{ \log (1 / \rho)\cdot\sqrt{T}} + \Lambda_4 \frac{\log^6 T}{(1-\rho)T}+ \frac{2}{T}.
\end{align}
The formula above is dominated by the first term on the right-hand side. Combining \eqref{eq:gapbound3} and Lemma~\eqref{gaplink}, we establish Theorem \ref{pe}.

\section{Technical Proofs of Lemmas in \S\ref{pf:pe}}

\subsection{Proof of Lemma \ref{value_func}}
Since the dynamic is linear, $V_{K}(x)$ has a quadratic form in $x$ specified as
\label{pf:value_func}
\begin{align}
\nonumber  V_{K}(x) &=\sum_{t=0}^{\infty}\left\{\mathbb{E}\left[c\left(x_{t}, u_{t}\right) | x_{0}=x\right]-C(K)\right\} \\ &=\sum_{t=0}^{\infty}\left\{\mathbb{E}\left[x_{t}^{\top}\left(Q+K^{\top} R K\right) x_{t}\right]+\sigma^{2} \cdot \tr(R)-C(K)\right\}. 
\end{align} 
By definition, we have $E_{x\sim \mathcal{D}_K}[V_{K}(x)] = 0$, so $V_{k}(x)=x^{\top} P_{K} x-E_{x\sim \mathcal{D}_K}\left[x^{\top} P_{K} x\right]$, for some matrix $P_{K}\in \R^{d\times d} $. We also have
\begin{equation}
V_{K}(x)=\mathbb{E}_{u \sim \pi_{K}}[c(x, u)]-C(K)+\mathbb{E}_{x^\prime}\left[V_{K}\left(x^{\prime}\right) \right],
\end{equation}
where $x^\prime$ is the next state generated by $(x,u)$. Therefore, we find that $P_K$ is a solution to the equation 
\begin{equation}
x^{\top} P_{K} x=x\left(Q+K^{\top} R K\right) x+x^{\top}(A-B K)^{\top} P_{K}(A-B K) x
\end{equation}
and the functional form of $V_{K}(x)$ follows from computing
\begin{align}
\mathbb{E}_{x \sim \mathcal{D}_{K}}\left[x^{\top} P_{K} x \right]  	\nonumber&=\tr\left[(Q+K^{\top} R K)\Sigma_{K}\right]+\tr\left[(A-B K)^{\top} P_{K}(A-B K)\Sigma_{K}\right]\\
\nonumber&= \tr\left[P_{K}\Phi_{\sigma}\right]+\tr\left[P_{K}(A-B K)\Sigma_{K}(A-B K)^{\top} \right]\\
&= \tr\left(P_{K} \Sigma_{K}\right).
\end{align}

Direct computation yields the functional form of $Q_{K}(x, u)$:
\begin{align} 
\nonumber Q_{K}(x, u) &=c(x, u)-C(K)+\mathbb{E}_{x^{\prime}}\left[V_{K}\left(x^{\prime}\right) | x\right] \\ 
\nonumber&=x^{\top} Q x+u^{\top} R u +(A x+B u)^{\top} P_{K}(A x+B u)+\tr\left(P_{K} \Phi\right)-C(K)-\tr\left(P_{K} \Sigma_{K}\right) \\
&=v\top \Delta_K v-\sigma^{2} \cdot \tr(R+P_{K} B B^{\top})-\tr\left(P_{K} \Sigma_{K}\right),
\end{align}
where the last equation follows from \eqref{eq:cost_K2},
$\Phi_{\sigma}=\Phi+\sigma^{2} \cdot B B^{\top}$, 
and the matrix $\Delta_K$ is given by 
\begin{equation}\label{eq:ThetaK}
\Delta_{K}=\left( \begin{array}{cc}{\Delta_{K}^{11}} & {\Delta_{K}^{12}} \\ {\Delta_{K}^{21}} & {\Delta_{K}^{22}}\end{array}\right)=\left( \begin{array}{cc}{Q+A^{\top} P_{K} A} & {A^{\top} P_{K} B} \\ {B^{\top} P_{K} A} & {R+B^{\top} P_{K} B}\end{array}\right).
\end{equation}
We define $\delta_{K}^{*} = \operatorname{svec}(\Delta_K)$. Then $\Delta_K$ can be recovered by $\Delta_K = \operatorname{smat}(\delta_{K}^{*})$. We can check that $E_K=\Delta_{K}^{22}K - \Delta_{K}^{21}$, thus estimating $\delta_{K}^{*}$ is equivalent to estimating $E_K$.

\subsection {Proof of Lemma \ref{lemma:compact_sets}}
\label{pf:compact_sets}

To prove $\gamma_{K}^{*} \in \mathcal{X}_{\Gamma}$, we just need to show $\|\delta_K^*\|_2=\|\Delta_K\|_F\leq\Gamma^2$. Note that we have
\begin{equation}
\Delta_{K}=\left( \begin{array}{cc}{Q+A^{\top} P_{K} A} & {A^{\top} P_{K} B} \\ {B^{\top} P_{K} A} & {R+B^{\top} P_{K} B}\end{array}\right) =\diag(Q,R)+\left(A \quad B\right) ^\top P_K \left(A \quad B\right).
\end{equation}
We have $\|\Delta_K\|_\mathrm{F}\leq\Gamma^2\leq \|Q\|_\mathrm{F}+\|R\|_\mathrm{F}+(\|A\|_\mathrm{F}^2+\|B\|_\mathrm{F}^2)\cdot\|P_K\|_\mathrm{F}$. 
From Lemma~\ref{bound_mats}, 
we have $\left\|P_{K}\right\| \leq C(K) / \sigma_{\min }(\Phi)$. 
Therefore, it holds that
\[
\left\|P_{K}\right\|_F\leq \sqrt{d}/ \sigma_{\min }(\Phi) \cdot C(K) \leq \sqrt{d}  / \sigma_{\min }(\Phi)\cdot C(K_0).
\]
Hence, we conclude 
\begin{align}
\|\Delta_K\|\leq\Gamma^2 &\leq \|Q\|_\mathrm{F}+\|R\|_\mathrm{F}+(\|A\|_\mathrm{F}^2+\|B\|_\mathrm{F}^2)\cdot\|P_K\|_\mathrm{F}\notag\\
&\leq  \|Q\|_\mathrm{F}+\|R\|_\mathrm{F}+(\|A\|_\mathrm{F}^2+\|B\|_\mathrm{F}^2)\cdot\sqrt{d}/ \sigma_{\min }(\Phi)\cdot C(K_0) ,
\end{align}
and thus we have $\gamma_{K}^{*} \in \mathcal{X}_{\Gamma}$.

To prove the second statement, observe that $\xi(\gamma)$ has components 
\begin{equation*}
\begin{aligned}
\xi(\gamma)^1 &= \gamma^{1}-C(K), \\
\xi(\gamma)^2 &= \gamma^{1} \cdot \mathbb{E}_{v}[\varphi(v)]+\Theta_{K} \gamma^{2}-d_{K}.
\end{aligned}
\end{equation*}
We bound 
$|\gamma^{1}-C(K)|$, 
$\|\gamma^{1} \cdot \mathbb{E}_{(x, u)}[\varphi (v)]\|_2$, 
$\|\Theta_{K} \gamma^{2}\|_2$ and $\|d_{K}\|_2$ separately. 
From the fact that $0 \leq \gamma^{1} \leq \Gamma^{1}=C\left(K_{0}\right)$, we have 
$|\gamma^{1}-C(K)|\leq C\left(K_{0}\right)$ and $0 \leq|\xi^{1}| \leq \Xi^{1}$.
From Lemma~\ref{lemma:chv}, we have 
\begin{align}
\|\gamma^{1} \cdot \mathbb{E}_{v}[\varphi(x, u)]\|_2 \leq C(K_0)\|\Breve{\Sigma}_{K}\|_\mathrm{F}\leq C(K_0)\left[k\sigma^2 +(d+\|K\|_{\mathrm{F}}^{2})\cdot\left\|\Sigma_{K}\right\|\right] .
\end{align}
From Lemma~\ref{lemma:pe_mat}, we have $\|\Theta_{K} \gamma^{2}\|_2\leq\|\Theta_{K}\|\|\gamma^{2}\|_2\leq4\left(1+\|K\|_{\mathrm{F}}^{2}\right)^{2} \cdot\left\|\Sigma_{K}\right\|^{2} \cdot  \Gamma^2$.
To bound $\|d_{K}\|_2$, note that for any positive definite matrix $\Sigma^\prime$, 
we can rewrite $d_K^\top \operatorname{svec}(\Sigma^\prime)$ by
\begin{align}
d_K^\top \operatorname{svec}(\Sigma^\prime) &= \mathbb{E}_{v}\{\langle\phi(v), \operatorname{smat}[\operatorname{diag}(Q, R)]\rangle \cdot\langle\phi(v), \operatorname{smat}(\Gamma)\rangle\notag\\
&= 2\left\langle\breve{\Sigma}_{K} \operatorname{diag}(Q, R) \breve{\Sigma}_{K}, \Gamma\right\rangle \cdot\left\langle\breve{\Sigma}_{K}, \operatorname{diag}(Q, R)\right\rangle \cdot\left\langle\breve{\Sigma}_{K}, \Gamma\right\rangle,
\end{align}
where the second equation follows from Lemma \ref{lemma:quadform} and 
$\breve{\Sigma}_{K}$ is defined in Lemma~\ref{lemma:chv}.
Thus, we conclude $\left\|d_{K}\right\|_{2} \leq 3\left(\|Q\|_{\mathrm{F}}+\|R\|_{\mathrm{F}}\right) \cdot\left\|\breve{\Sigma}_{K}\right\|^{2}$.
Combining the bounds above and \eqref{Ebound}, we conclude that for some constant $C$, 
\begin{align}
\left\|\xi^{2}\right\|_{2} &\leq C(K_0)\left[k\sigma^2 +(d+\|K\|_{\mathrm{F}}^{2})\cdot\left\|\Sigma_{K}\right\|\right]+4\left(1+\|K\|_{\mathrm{F}}^{2}\right)^{2} \cdot\left\|\Sigma_{K}\right\|^{2} \cdot  \Gamma^2 \notag\\
&\quad+ 12\left(\|Q\|_{\mathrm{F}}+\|R\|_{\mathrm{F}}\right) \cdot\left(d+\|K\|_{\mathrm{F}}^{2}\right)^{2} \cdot\left\|\Sigma_{K}\right\|^{2}\notag\\
&\leq C \cdot  \left(1+\|K\|_{\mathrm{F}}^{2}\right)^{2}\cdot \Gamma^2 \cdot \sigma_{\min}^{-2} (Q) \cdot [C(K_0)]^2 .
\end{align}

\subsection {Proof of Lemma \ref{lemma:esitimate_bound}}
\label{pf:esitimate_bound}
Note that the optimal value of the dual problem satisfies 
\begin{align} 
G_{\hat\gamma\cdot} =\max _{\xi \in \mathcal{X}_{\Xi}} G(\hat\gamma, \xi)  &= \|\widehat{\gamma}^{1}-C(K)\|^{2}+\left\|\widehat{\gamma}^{1} \cdot \mathbb{E}_{(x, u)}[\varphi(x, u)]+\Theta_{K} \hat\gamma^{2}-d_{K}\right\|_{2}^{2} \notag\\
&= \|\Omega_K \hat\gamma-b_K\|_2^2 = \|\Omega_K (\hat\gamma-\gamma_{K}^{*})\|_2^2,
\end{align} 
where we define 
\begin{equation}\label{eq:Ometa_gamma}
\Omega_K = \left( \begin{array}{cc}{1} & {0} \\ {\mathbb{E}_{v}[\varphi(v)]} & {\Theta_{K}}\end{array}\right), \quad \gamma_{K}^{*} = \left(C(K), \delta_{K}^{* \top}\right)^{\top}.
\end{equation}
The optimal value of the primal value satisfies
\begin{align} 
G_{\cdot\hat\xi} =\min _{\gamma \in \mathcal{X}_{\Gamma}} G(\gamma, \hat\xi)  \leq  \min _{\gamma \in \mathcal{X}_{\Gamma}} G_{\hat\gamma\cdot} = \min _{\gamma \in \mathcal{X}_{\Gamma}} \|\Omega_K \hat\gamma-b_K\|_2^2 = 0.
\end{align} 
Therefore $\|\Omega_K (\hat\gamma-\gamma_{K}^{*})\|_2^2 \leq G_{\hat\gamma\cdot} - G_{\cdot\hat\xi}$.
Moreover, by Lemma \ref{lemma:pe_mat}, we have $\|\Omega_K\|\geq \tau_{K}^{*} > 0 $ for some constant $\tau_{K}^{*}$ that only depends on $\rho(A-B K)$, $\sigma$, and $\sigma_{\min }(\Phi)$.
Hence, we have
\begin{equation}
(\tau_{K}^{*})^2 \left( |\gamma^1-C(K)|^2 + \|\hat\Delta_K - \Delta_K\|^2_F\right)\leq \tau_{K}^{*} \|\hat\gamma-(\tau_{K}^{*})^2\|_2^2 \leq \|\Omega_K (\hat\gamma-\gamma_{K}^{*})\|_2^2 \leq G_{\hat\gamma\cdot} - G_{\cdot\hat\xi}.
\end{equation}

\subsection {Proof of Lemma \ref{lemma:tuncation_prob}}
\label{pf:tuncation_prob}
Set $s = C^\prime \cdot \log T \cdot\left\|\Breve{\Sigma}_{K}\right\|$ with $C^\prime$ sufficiently 
large so that $s^{2} \cdot\|\Breve{\Sigma}_{K}\|_{\mathrm{F}}^{-2}\geq s \cdot\|\Breve{\Sigma}_{K}\|^{-1}$.
Applying Lemma \ref{Hansen-Wright} to $v\sim N\left(0, \Breve{\Sigma}_{K}\right)$, we get
\begin{align}
\mathbb{P}\left[\|v\|_{2}^{2}-\tr(\Breve{\Sigma}_{K}) |>s\right] 
&\leq 2 \cdot \exp [-C \cdot \min(s^{2} \cdot\|\Breve{\Sigma}_{K}\|_{\mathrm{F}}^{-2}, s \cdot\|\Breve{\Sigma}_{K}\|^{-1})]\notag\\
&=2 \cdot \exp [-C  \cdot\|\Breve{\Sigma}_{K}\|^{-1}\cdot s]\notag\\
&=2 \cdot \exp [-C  \cdot C^\prime \cdot \log T ].
\end{align}
When $C^\prime$ is sufficiently large, we have $\mathbb{P}\left(\mathcal{A}^{c}\right) \leq T^{-6}$.

\subsection {Proof of Lemma \ref{lemma:bound_obj_diff_tunc}}
\label{pf:bound_obj_diff_tunc}

By direct computation, we have
\begin{align}
G(\gamma, \xi)-\widetilde{G}(\gamma, \xi) &=\notag \left[\gamma^{1}-\mathbb{E}_{v}[c(v)]\right] \cdot \xi^{1} - \left[\gamma^{1}-\mathbb{E}_{v}[c(v)\cdot\ind_{\mathcal{A}}]\right] \cdot \xi^{1} \\
&\notag\quad+\left\langle\gamma^{1} \cdot  \mathbb{E}_{v}[\varphi(v)]+\Theta_{K} \gamma^{2}- d_{K}, \xi^{2}\right\rangle - \left\langle\gamma^{1} \cdot  \mathbb{E}_{v}[\tilde\varphi(v)]+\tilde\Theta_{K} \gamma^{2}-\tilde d_{K}, \xi^{2}\right\rangle \\
&=\notag \mathbb{E}_{v}\left[c(v)\cdot\ind_{\mathcal{A}^{c}}\right]\cdot\xi^{1} - \left\langle\mathbb{E}_{v}\left[\varphi(v)\varphi(v^\prime)^\top\ind_{\mathcal{A}}\ind_{\mathcal{A^\prime}^{c}}\right] \gamma^{2}, \xi^{2}\right\rangle\\
&\notag\quad+ \left\langle\mathbb{E}_{v}\left[\gamma^{1} \cdot  \varphi(v)\ind_{\mathcal{A}^{c}}+\varphi(v)[\varphi(v)-\varphi(v^\prime)]^\top\ind_{\mathcal{A}^{c}} \gamma^{2}- c(v)\varphi(v)\ind_{\mathcal{A}^{c}}\right], \xi^{2}\right\rangle\\
&=\notag \mathbb{E}_{v}\left[c(v)\cdot\ind_{\mathcal{A}^{c}}\right]\cdot\xi^{1} - \left\langle\mathbb{E}_{v}\left[\varphi(v)\varphi(v^\prime)^\top\ind_{\mathcal{A}}\ind_{\mathcal{A^\prime}^{c}}\right] \gamma^{2}, \xi^{2}\right\rangle\\
&\quad+ \left\langle\mathbb{E}_{v}\left[ \varphi(v)[\gamma^{1}-c(v)+\varphi(v)^\top\gamma^2]\ind_{\mathcal{A}^{c}} -\varphi(v)\varphi(v^\prime)^\top\gamma^{2}\ind_{\mathcal{A}^{c}} \right], \xi^{2}\right\rangle
\end{align}
and
\#
|G(\gamma, \xi)-\widetilde{G}(\gamma, \xi)|&\leq\left|\mathbb{E}_{v}\left[c(v)\cdot\ind_{\mathcal{A}^{c}}\right]\right|\cdot C\left(K_{0}\right)+\left\|\mathbb{E}_{v}\left[\varphi(v)\varphi(v^\prime)^\top\ind_{\mathcal{A}}\ind_{\mathcal{A^\prime}^{c}}\right] \gamma^{2}\right\|\cdot\Xi^{2} \notag\\
&\quad+ \left\|\mathbb{E}_{v}\left[\gamma^{1} \cdot  \varphi(v)\ind_{\mathcal{A}^{c}}+\varphi(v)[\varphi(v)-\varphi(v^\prime)]^\top\ind_{\mathcal{A}^{c}} \gamma^{2}- c(v)\varphi(v)\ind_{\mathcal{A}^{c}}\right]\right\|\cdot\Xi^{2}.
\#
We can bound each term on the right-hand side separately with
\begin{equation}
\left(\left|\mathbb{E}_{v}\left[c(v)\cdot\ind_{\mathcal{A}^{c}}\right]\right| \leq \mathbb{E}\left[c^{2}(v)\right]\right)^{\frac{1}{2}}\cdot \mathbb{P}\left(\mathcal{A}^{c}\right)^{\frac{1}{2}},
\end{equation}
\begin{equation}
\left\|\mathbb{E}_{v}\left[\varphi(v)\varphi(v^\prime)^\top\ind_{\mathcal{A}}\ind_{\mathcal{A^\prime}^{c}}\right] \gamma^{2}\right\| \leq \left(\mathbb{E}_{v}\left[\|\varphi(v)\varphi(v^\prime)^\top\|^2_2\right]\right)^{\frac{1}{2}}\cdot\mathbb{P}\left(\mathcal{A}^{c}\right)^{\frac{1}{2}},
\end{equation}
where
\begin{equation}\label{eq:momentbound1}
\mathbb{E}_{v}\left[\|\varphi(v)\varphi(v^\prime)^\top\gamma^2\|^2_2\right]\leq \left(\mathbb{E}\left[|\varphi\left(v^{\prime}\right)^{\top} \gamma^2|^{4}\right] \cdot \mathbb{E}\left[\|\varphi(v)\|_{2}^{4}\right]\right)^{\frac{1}{2}}.
\end{equation}
Furthermore, 
\begin{align}
&\left\|\mathbb{E}_{v}\left[ \varphi(v)[\gamma^{1}-c(v)+\varphi(v)^\top\gamma^2]\ind_{\mathcal{A}^{c}} -\varphi(v)\varphi(v^\prime)^\top\gamma^{2}\ind_{\mathcal{A}^{c}} \right]\right\|\notag\\
&\qquad \leq \left(\mathbb{E}_{v}\left[\left\| \varphi(v)[\gamma^{1}-c(v)+\varphi(v)^\top\gamma^2] -\varphi(v)\varphi(v^\prime)^\top\gamma^{2} \right\|^2\right] \right)^{\frac{1}{2}}\cdot\mathbb{P}\left(\mathcal{A}^{c}\right)^{\frac{1}{2}},
\end{align}
where
\#\label{eq:momentbound2}
&\mathbb{E}_{v}\left[\left\| \varphi(v)[\gamma^{1}-c(v)+\varphi(v)^\top\gamma^2] -\varphi(v)\varphi(v^\prime)^\top\gamma^{2} \right\|^2\right]\notag\\
& \leq 2\mathbb{E}_{v}\left[\left| \gamma^{1}-c(v)+\varphi(v)^\top\gamma^2 \right|^2 \left\|\varphi(v)\right\|^2+\left\|\varphi(v)\varphi(v^\prime)^\top\gamma^{2} \right\|^2\right]\notag\\
& \leq 2\left(\mathbb{E}_{v} \left[|\gamma^{1}-c(v)+\varphi(v)^\top\gamma^2 |^4\right]\cdot \mathbb{E}_{v}\left[\left\|\varphi(v)\right\|^4\right]\right)^{\frac{1}{2}}+2\left(\mathbb{E}_{v}\left[|\varphi\left(v^{\prime}\right)^{\top} \gamma^2|^{4}\right] \cdot \mathbb{E}_{v}\left[\|\varphi(v)\|_{2}^{4}\right]\right)^{\frac{1}{2}}.
\#
Note that \eqref{eq:momentbound1} and \eqref{eq:momentbound2} 
can be bounded by the fourth moments of $N\left(0, \breve{\Sigma}_{K}\right)$.
Since $P(\mathcal{A}^{c}) \leq T^{-6}$, when $T$ is sufficiently large, 
it holds that $|G(\gamma, \xi)-\widetilde{G}(\gamma, \xi)| \leq 1 / T$.

\subsection {Proof of Lemma \ref{lemma:lip}}
\label{pf:lip}
By direct computation, we have 
\begin{align} 
\nabla_{\gamma^{1}} G\left(\gamma, \xi \right)&=\xi^{1}+\mathbb{E}_{v}[\widetilde{\varphi}(v)^{\top}] \xi^{2},\\
\nabla_{\gamma^{2}} G\left(\gamma, \xi \right)&=\mathbb{E}_{v}\left[\widetilde{\varphi}(v)^{\top} \xi^{2} \cdot[\widetilde{\varphi}(v)-\widetilde{\varphi}\left(v^{\prime}\right)]\right],\\
\nabla_{\xi^{1}} G\left(\gamma, \xi \right)&=\gamma^{1}-\mathbb{E}_{v}[\widetilde{c}(v)]-\xi^{1},\\
\nabla_{\xi^{2}} G\left(\gamma, \xi \right)&=\mathbb{E}_{v}\left[\gamma^1\widetilde{\varphi}(v)+\widetilde{\varphi}(v)(\widetilde{\varphi}(v)-\widetilde{\varphi}(v^\prime))^\top\gamma^2-\widetilde{c}(v)\widetilde{\varphi}(v)\right] -\xi^{2}.
\end{align} 
By definition, we have the bound
\#\label{eq:truncbound}
\|\tilde{\varphi}(v)\|_2 
&\leq C^\prime \cdot \log T \cdot\|\Breve{\Sigma}_{K}\|+\tr(\Breve{\Sigma}_{K})\notag\\
&\leq (C^\prime \cdot \log T+d+k)\|\Breve{\Sigma}_{K}\| \notag\\
&\leq (C^\prime \cdot \log T+d+k)[\sigma^2+(1+\|K\|_F^2)\|\Sigma_K\|],
\#
where the last inequality follows from \eqref{eq:Breve_Sigma_K_bound}.
The same bound holds for $\|\tilde{\varphi}(v^\prime)\|_2$. 

Combining \eqref{eq:truncbound} and definitions of $\mathcal{X}_{\Gamma}$ and $\mathcal{X}_{\Xi}$, by direct computation, we have
\begin{align}
\|\nabla_{\gamma} \widetilde{G}(\gamma, \xi)\|_{2}
&\leq C(K_0)+ \|\widetilde{\varphi}(v)\|_2 \cdot\|\xi^{2}\|_2+\|\widetilde{\varphi}(v)\|_2 \cdot\|\xi^{2}\|_2\cdot\|\tilde{\varphi}(v)+\tilde{\varphi}(v^{\prime})\|_2\notag\\
&\leq C(K_0)+ \Xi^2 \cdot(C^\prime \cdot \log T+d+k)[\sigma^2+(1+\|K\|_F^2)\|\Sigma_K\|]  \notag\\
&\quad + 2\cdot\Xi^2 \cdot(C^\prime \cdot \log T+d+k)^2[\sigma^2+(1+\|K\|_F^2)\|\Sigma_K\|]^2.
\end{align}\
Similarly, we have
\begin{align}
\|\nabla_{\xi} \widetilde{G}(\gamma, \xi)\|_{2}
&\leq 2C(K_0)+ C(K_0) (C^\prime \cdot \log T+d+k)[\sigma^2+(1+\|K\|_F^2)\|\Sigma_K\|]\notag\\
&\quad+2 \cdot \Gamma^2 \cdot(C^\prime \cdot \log T+d+k)^2[\sigma^2+(1+\|K\|_F^2)\|\Sigma_K\|]^2\cdot \notag\\
&\quad+  C(K_0) (C^\prime \cdot \log T+d+k)[\sigma^2+(1+\|K\|_F^2)\|\Sigma_K\|]+\Xi^2.
\end{align}
Using that fact that $\sqrt{x+y}<x + y $ when $x>0$ and $y>0$, we have
\begin{align}
\|\nabla \widetilde{G}(\gamma, \xi)\|_{2}
&\leq	\|\nabla_{\gamma} \widetilde{G}(\gamma, \xi)\|_{2}+\|\nabla_{\xi} \widetilde{G}(\gamma, \xi)\|_{2}\notag \\
& \leq 3 (\Gamma^2+\Xi^2)\cdot(C^\prime \cdot \log T+d+k)^2[\sigma^2+(1+\|K\|_F^2)\|\Sigma_K\|]^2,
\end{align}
where $C_1$ is sufficiently large.

\subsection {Proof of Lemma \ref{lemma:primaldual}}
\label{pf:primaldual}

We can further specify $L_1$, $L_2$, $D$ and $C_\zeta$.
By Lemma~\ref{lemma:lip}, we can set 
\begin{equation}
L_1 = 3 (\Gamma^2+\Xi^2)\cdot(C^\prime \cdot \log T+d+k)^2[\sigma^2+(1+\|K\|_F^2)\|\Sigma_K\|]^2, \quad L_2 = 1.
\end{equation}
By \eqref{projsets}, we can set $D$ to be
\begin{align}
D= (2C(K_0)^2 + (\Gamma^{2})^2 + (\Xi^{2})^2)^{1/2}.
\end{align}
By Lemma \ref{lemma:chv} and \ref{lemma:beta_mixing}, we can set 
\begin{align}
C_{\zeta}=C_{\rho, \breve{K}} \cdot\left[\tr\left(\Breve{\Sigma}_{K}\right)+(d+k) \cdot(1-\rho)^{-2}\right]^{1 / 2}.
\end{align}
Set $\delta = T^{-5}$ and note that $\log x<x+1$ for $\forall x >0$.
We conclude that with probability at least $1-T^{-5}$ the primal-dual gap conditioned on $\bigcap_{t=1}^{T}\mathcal{A}_t$ is bounded by 
\begin{align}\label{eq:gapbound1}
\tilde G_{\hat\gamma\cdot} - \tilde G_{\cdot\hat\xi} \leq \Lambda_3\frac{{\log^2 T}+\log(T^5)}{ \log (1 / \rho)\cdot\sqrt{T}} + \Lambda_4 \frac{\log^6 T}{(1-\rho)T},
\end{align}
when $T$ is sufficiently large, where $\Lambda_3$ and $\Lambda_4$ are polynomials of $\Gamma^2$, $\Xi^2$, $C(K_0)$, $\|K\|_{\mathrm{F}}$, $\Sigma_K$, and $\sigma_{\min }^{-1}(Q)$ .

\section {Proof of Supporting Lemmas}
\label{sec:supp_pf}

In this section, we lay out the proofs of supporting lemmas.

\begin{lemma}\label{lemma:cost_grad_K_single}
	Consider the LQR problem specified by
	\#\label{lqr_single}
	x_{t+1}=A  x_{t} + B u_{t} + w_t, \quad c\left(x_{t}, u_{t}\right)=x_{t}^{\top} Q x_{t}+u_{t}^{\top} R u_{t},
	\#
	where $A$, $B$, $Q$ and $R$ are matrices of proper dimensions, and noise $w_t \sim N(0,\Phi)$, with matrices $Q, R, \Phi \succ 0$. Under a policy $\pi_{K}$ that satisfies $\rho(A-B K)<1$, an action $u_{t}$ is written as $-K x_{t}+\sigma \cdot z_{t}$, where $z_{t} \sim N\left(0, I_{d}\right)$ is a Gaussian noise independent of $x_t$ used to encourage exploration. By direct computation, the state dynamic is given by 
	\#
	x_{t+1}=( A- B  K)  x_{t}+  \varepsilon_{t}, \quad  \varepsilon_{t}\sim N\left(0,  \Phi_{\sigma}\right),
	\# 
	where $\Phi_{\sigma}:=\Phi+\sigma^{2} \cdot B B^{\top}$. We denote by $\Sigma_{ K}$ the unique positive definite solution to the Lyapunov equation
	\#
	\Sigma_{ K}=\Phi_{\sigma}+( A- B  K) \Sigma_{ K}( A- B  K)^{\top}.
	\#
	Then \eqref{lqr_single} has a stationary state distribution $N(0;\Sigma_K)$.
	We further denote by $P_{K}$ the unique positive definite solution to the Lyapunov equation
	\#
	P_{K}=\left(Q+K^{\top} R K\right)+(A-B K)^{\top} P_{K}(A-B K).
	\#
	Then the corresponding time-average cost and its gradient under policy $\pi_{K}$ are given, respectively, by
	\begin{align}\label{eq:cost_K2_single}
	C( K)&=\tr\left[\left(Q+K^{\top} R K\right) \Sigma_{K}\right]+\sigma^{2} \cdot \tr(R)=\tr\left(P_{K} \Phi_{\sigma}\right)+\sigma^{2} \cdot \tr(R),\\
	\nabla_{K}  C(K)&=2\left[\left(R+B^{\top} P_{K} B\right) K-B^{\top} P_{K} A\right] \Sigma_{K}=2 E_{K} \Sigma_{K},\label{eq:cost_K2_single_gradient}
	\end{align}
	where we define $E_{K} =\left(R+B^{\top} P_{K} B\right) K-B^{\top} P_{K} A$.
\end{lemma}

\begin{proof}\label{pf:cost_grad_K}
	For all $t \geq 0$, we have 
	\begin{align} 
	\mathbb{E}\left[c\left(x_{t}, u_{t}\right) | x_{t}\right] &=x_{t}^{\top} Q x_{t}+\mathbb{E}_{z_{t} \sim N\left(0, I_{d}\right)}\left[\left(-K x_{t}+\sigma \cdot z_{t}\right)^{\top} R\left(-K x_{t}+\sigma \cdot z_{t}\right)\right] \notag\\ &=x_{t}^{\top}\left(Q+K^{\top} R K\right) x_{t}+\sigma^{2} \cdot \tr(R) .
	\end{align}
	Then the time-average cost is given by
	\begin{align} 
	C(K)
	&=\mathbb{E}_{x \sim \mathcal{D}_{K}}\left[x^{\top}\left(Q+K^{\top} R K\right) x\right]+\sigma^{2} \cdot \tr(R)\notag\\
	&=\tr\left[\left(Q+K^{\top} R K\right) \Sigma_{K}\right]+\sigma^{2} \cdot \tr(R),
	\end{align}
	where $\mathcal{D}_{K}$ is the stationary distribution of the Markov 
	chain$\left\{x_{t}\right\}_{t\geq0}$. 
	
	To see why the second equation in 
	\eqref{eq:cost_K2_single} holds, consider operators $\mathcal{T}_{K}$ and $\mathcal{T}_{K}^{\top}$ defined by
	\begin{equation}\label{eq:operator}
	\mathcal{T}_{K}(X)=\sum_{t \geq 0}(A-B K)^{t} X\left[(A-B K)^{t}\right]^{\top}, \quad \mathcal{T}_{K}^{\top}(X)=\sum_{t \geq 0}\left[(A-B K)^{t}\right]^{\top} X(A-B K)^{t},
	\end{equation}
	where $X$ is a positive definite matrix of proper dimension. By direct computation, we have $\tr\left[X_{1} \cdot \mathcal{T}_{K}\left(X_{2}\right)\right]$ = $\tr\left[\mathcal{T}_{K}^{\top}\left(X_{1}\right) \cdot X_{2}\right]$ for positive definite matrices $X_1$ and $X_2$. 
	Then the proof is concluded by observing that
	$\Sigma_{K}=\mathcal{T}_{K}\left(\Phi_{\sigma}\right)$ and 
	$P_{K}=\mathcal{T}_{K}^{\top}(Q+K^{\top} R K )$.
\end{proof}

\begin{lemma}\label{lemma:cost_diff}
	Following the notation in Lemma~\ref{lemma:cost_grad_K_single}, we consider the LQR problem specified in \eqref{lqr_single}. Let $K$ and $K^\prime$ be stable policies that satisfy $\rho(A-B K)<1$ and $\rho\left(A-B K^{\prime}\right)<1$. Let $\left\{x_{t}^{\prime}\right\}_{t \geq 0}$ be the state sequence induced by $x^\prime_{t+1}=(A-BK^\prime)x^\prime_t$ with $x_{0}^{\prime}=x\in\mathbb{R}^d$. Then we have
	\begin{equation}
	\label{eq:advantage_sum} 
	x^{\top} P_{K^{\prime}} x-x^{\top} P_{K} x = 
	\sum_{t \geq 0} A_{K, K^{\prime}}\left(x_{t}^{\prime}\right), 
	\end{equation}
	where we define
	\begin{align}
	A_{K, K^{\prime}}(x) \coloneqq 2 x^{\top}\left(K^{\prime}-K\right)^{\top} E_{K} x+x^{\top}\left(K^{\prime}-K\right)^{\top}\left(R+B^{\top} P_{K} B\right)\left(K^{\prime}-K\right) x\notag.
	\end{align}
	Moreover, it holds that 
	\begin{equation}\label{eq:lower_bound_advantage}
	A_{K, K^{\prime}}(x) \geq-\tr\left[x x^{\top} E_{K}^{\top}\left(R+B^{\top} P_{K} B\right)^{-1} E_{K}\right].
	\end{equation}
\end{lemma}

\begin{proof}\label{pf:cost_diff}
	The inequality \eqref{eq:lower_bound_advantage} follows by direct computation:
	\begin{align}
	A_{K, K^{\prime}}(x) &=\tr\Bigg\{x x^{\top}\left[K^{\prime}-K+\left(R+B^{\top} P_{K} B\right)^{-1} E_{K}\right]^{\top}\left(R+B^{\top} P_{K} B\right) \notag\\
	&\quad\left[K^{\prime}-K+\left(R+B^{\top} P_{K} B\right)^{-1} E_{K}\right]\Bigg\} -\tr\left[x x^{\top} E_{K}^{\top}\left(R+B^{\top} P_{K} B\right)^{-1} E_{K}\right] \notag\\
	&\geq-\tr\left[x x^{\top} E_{K}^{\top}\left(R+B^{\top} P_{K} B\right)^{-1} E_{K}\right].
	\end{align}
	To prove \eqref{eq:advantage_sum}, note that $P_{K^\prime}$ satisfies the equation 
	\[
	P_{K^\prime}=\left(Q+K^{\prime\top} R K^{\prime\top}\right)+(A-B K^{\prime\top})^{\top} P_{K^{\prime}}(A-B K^{\prime\top}).
	\] 
	By direct computation, we have
	\begin{align}
	x^{\top} P_{K^{\prime}} x-x^{\top} P_{K} x &= \sum_{t \geq 0} x^{\top}\left[\left(A-B K^{\prime}\right)^{t}\right]^{\top}\left(Q+K^{\prime \top} R K^{\prime}\right)\left[\left(A-B K^{\prime}\right)^{t}\right] x - x^{\top} P_{K} x\notag \\
	&=\sum_{t \geq 0} x_{t}^{\prime \top}\left(Q+K^{\prime\top} R K^{\prime}\right) x_{t}^{\prime} - x^{\top} P_{K} x \notag \\
	&=\sum_{t \geq 0} \left[x_{t}^{\prime \top}\left(Q+K^{\prime\top} R K^{\prime}\right) x_{t}^{\prime} +x_{t}^{\prime \top} P_{K} x_{t}^{\prime}-x_{t}^{\prime \top} P_{K} x_{t}^{\prime}\right] - x_{0}^{\prime\top} P_{K} x_{0}^{\prime} \notag \\
	&=\sum_{t \geq 0} \left[x_{t}^{\prime \top}\left(Q+K^{\prime\top} R K^{\prime}\right) x_{t}^{\prime} +x_{t+1}^{\prime \top} P_{K} x_{t+1}^{\prime}-x_{t}^{\prime \top} P_{K} x_{t}^{\prime}\right] .
	\end{align}
	To see how the last line is related to the value $\sum_{t \geq 0} A_{K, K^{\prime}}\left(x_{t}^{\prime}\right)$, we have
	\begin{align}
	x^{\top} Q x &+ \left(-K^{\prime} x\right)^{\top} R\left(-K^{\prime} x\right)+\left[\left(A-B K^{\prime}\right) x\right]^{\top} P_{K}\left[\left(A-B K^{\prime}\right) x\right]-x^{\top} P_{K} x \notag \\
	&=x^{\top}\left[Q+\left(K^{\prime}-K+K\right)^{\top} R\left(K^{\prime}-K+K\right)\right] x \notag \\&\quad+ x^{\top}\left[A-B K-B\left(K^{\prime}-K\right)\right]^{\top} P_{K}\left[A-B K-B\left(K^{\prime}-K\right)\right] x-x^{\top} P_{K} x \notag \\
	&= 2 x^{\top}\left(K^{\prime}-K\right)^{\top}\left[\left(R+B^{\top} P_{K} B\right) K-B^{\top} P_{K} A\right] x \notag\\
	&\quad+x^{\top}\left(K^{\prime}-K\right)^{\top}\left(R+B^{\top} P_{K} B\right)\left(K^{\prime}-K\right) x.  
	\end{align}
	Recalling that $E_K = \left(R+B^{\top} P_{K} B\right) K-B^{\top} P_{K} A$, we conclude \eqref{eq:advantage_sum}. 
\end{proof}

\begin{lemma}[Perturbation of $C(K)$]\label{lemma:perturbation}
	Suppose $K^\prime$ is a perturbation of $K$ and satisfies
	\#\label{eq:policyboundcondition}
	\left\|K^{\prime}-K\right\| \leq \sigma_{\min }(\Phi) / 4 \cdot\left\|\Sigma_{K}\right\|^{-1}\|B\|^{-1} \cdot(\|A-B K\|+1)^{-1}.
	\#
	Then it holds that
	\#\label{eq:costperturbationbound}
	\left\|C({K^{\prime}})-C({K})\right\| &\leq 6 \left(\|\Phi\| +\sigma^{2} \cdot\|B\|^{2}\right)\cdot \sigma_{\min }^{-1}(\Phi) \cdot\left\|\Sigma_{K}\right\| \cdot\|K\| \cdot\|R\| \notag\\
	&\qquad \qquad \cdot(\|K\| \cdot\|B\| \cdot\|A-B K\|+\|K\| \cdot\|B\|+1) \cdot\left\|K^{\prime}-K\right\|.
	\#
\end{lemma}
\begin{proof}
	This lemma is obtained by combining Lemmas 17 and 24 in \cite{fazel2018global}. We sketch the proof below. First, we have the inequality 
	\#\label{eq:bound_C}
	\left|C\left(K\right)-C\left(K^{\prime}\right)\right|&=\left|\tr\left[\left(P_{K}-P_{K^{\prime}}\right) \cdot \Phi_{\sigma}\right]\right| \leq\left\|\Phi_{\sigma}\right\| \cdot\left\|P_{K}-P_{K^{\prime}}\right\|\notag\\
	&\leq\left[\|\Phi\|+\sigma^{2} \cdot\|B\|^{2}\right] \cdot\left\|P_{K}-P_{K^{\prime}}\right\|.
	\#
	Under condition \eqref{eq:policyboundcondition}, by Lemma 24 in \cite{fazel2018global}, we can bound $\left\|P_{K}-P_{K^{\prime}}\right\|$ as
	\#\label{eq:bound_P}
	\left\|P_{K^{\prime}}-P_{K}\right\| \leq 6\left\|\mathcal{T}_{K}\right\| \cdot\|K\| \cdot\|R\| \cdot\|B\| \cdot\|A-B K\|+\|K\| \cdot\|B\|+1 ) \cdot\left\|K^{\prime}-K\right\|,
	\#
	where the operator $\mathcal{T}_{K}$ is defined in \eqref{eq:operator}.
	By Lemma 17 in \cite{fazel2018global}, we can further bound $\left\|\mathcal{T}_{K}\right\|$ as 
	\#\label{eq:bound_T}
	\left\|\mathcal{T}_{K}\right\| \leq \sigma_{\min }^{-1}(\Phi) \cdot\left\|\Sigma_{K}\right\|.
	\#
	Combining \eqref{eq:bound_C}, \eqref{eq:bound_P} and \eqref{eq:bound_T} completes the proof.
\end{proof}

\begin{lemma}
	\label{bound_mats}
	Let $\pi_{K}$ be a stable policy such that $\rho(A-BK)<1$. Then it holds that 
	\begin{equation}
	\left\|\Sigma_{K}\right\| \leq C(K) / \sigma_{\min }(Q), \quad\left\|P_{K}\right\| \leq C(K) / \sigma_{\min }(\Phi).
	\end{equation}
\end{lemma}	
\begin{proof}
	This follows from Lemma~\ref{lemma:cost_grad_K} and direct computation
	\begin{align}
	C(K) \geq& \tr\left[\left(Q+K^{\top} R K\right) \Sigma_{K}\right] \geq \sigma_{\min }(Q) \cdot \tr\left(\Sigma_{K}\right) \geq \sigma_{\min }(Q) \cdot\left\|\Sigma_{K}\right\|,\\
	C(K) \geq& \tr\left(P_{K} \Phi_{\sigma}\right) \geq \sigma_{\min }\left(\Phi_{\sigma}\right) \cdot \tr\left(P_{K}\right) \geq  \sigma_{\min }\left(\Phi_{\sigma}\right) \left\|P_{K}\right\|. 
	\end{align}
\end{proof}

\begin{lemma}[Characterization of $v$]\label{lemma:chv}
	Let $v=\left[x^{\top}, u^{\top}\right]^{\top}$ be defined as in Lemma \ref{value_func}. 
	Then $\{v_t\}_{t\geq0}$ is a linear system with transition equation 
	$v^{\prime}=\breve{K} v+\breve{w}$, where 
	\[
	\breve{K}=\left( \begin{array}{cc}{A} & {B} \\ {-K A} & {-K B}\end{array}\right), \qquad
	\breve{\varepsilon}=\left( \begin{array}{c}{w} \\ {-K w+\sigma \cdot z}\end{array}\right).
	\]
	Moreover, $v$ has a stationary distribution $N\left(0, \Breve{\Sigma}_{K}\right)$, where 
	\begin{align}
	\Breve{\Sigma}_{K}=\left( \begin{array}{cc}{\Sigma_{K}} & {-\Sigma_{K} K^{\top}} \\ {-K \Sigma_{K}} & {K \Sigma_{K} K^{\top}+\sigma^{2} \cdot I_{k}}\end{array}\right)
	\end{align}
	with
	\begin{align}
	\|\Breve{\Sigma}_{K}\|_{\mathrm{F}} & \leq k\sigma^2 +(d+\|K\|_{\mathrm{F}}^{2})\cdot \|{\Sigma}_{K}\|,\\
	\|\Breve{\Sigma}_{K}\| & \leq \sigma^2 +(1+\|K\|_{\mathrm{F}}^{2})\cdot \|{\Sigma}_{K}\|.
	\end{align}
	Furthermore,  $\{v_t\}_{t\geq0}$ is a geometrically $\beta$-mixing stochastic process with the parameter $\rho \in (\rho(A-BK),1)$.
\end{lemma}
\begin{proof}\label{pf:chv}
	Let $v^\prime=\left[x^{\prime\top}, u^{\prime\top}\right]^{\top}$ be the next state and action.	The transition is given by 
	\begin{align}
	x^{\prime} &= A x+B u+w,\\
	u^{\prime} &= -K x^{\prime}+\sigma \cdot z = -K A x-K B u-K w+\sigma \cdot z.
	\end{align}
	Then $v^\prime= \Breve K v + \Breve w$ and $v$ has a Gaussian distribution with the covariance matrix 
	\begin{align}\Breve \Phi_K = \left( \begin{array}{cccc}{\Phi} & {-\Phi K^{\top}} \\ {-K \Phi} & {K \Phi K^{\top}+\sigma^{2} \cdot I_{k}}\end{array} \right). 
	\end{align}
	Moreover, note that
	\[
	\Breve K=\left( \begin{array}{cc}{A} & {B} \\ {-K A} & {-K B}\end{array}\right)=\left( \begin{array}{c}{I_{d}} \\ {-K}\end{array}\right) \left( \begin{array}{cc}{A} & {B}\end{array}\right)
	\]
	with the spectral norm bounded as $\rho(\breve K)=\rho(A- BK)<1$.

	We can find the stationary distribution of $v$ by solving the Lyapunov equation. 
	By direct computation, we can check that $\Breve{\Sigma}_{K}$ is the unique solution to 	the Lyapunov equation $\Breve{\Sigma}_{K}=\Breve K \Breve{\Sigma}_{K} \Breve K^{\top}+\Breve{\Phi}_{K}$. Therefore, $v$ has a stationary distribution $N(0,\widetilde{\Sigma}_{K})$.
	
	To bound $\|\Breve{\Sigma}_{K}\|$, note that we have
	\[
	\Breve{\Sigma}_{K}=\left( \begin{array}{cc}{0} & {0} \\ {0} & {\sigma^{2} I_{k}}\end{array}\right)+\left( {I_{d}} \quad{-K}\right)^{\top} \Sigma_{K} \left( {I_{d}}\quad  {-K}\right).
	\]
	By direct computation, we have
	\begin{align}\label{eq:Breve_Sigma_K_bound_F}
	\|\Breve{\Sigma}_{K}\|_\mathrm{F} &\leq k\sigma^2 +(d+\|K\|_{\mathrm{F}}^{2})\cdot \|{\Sigma}_{K}\| , \\
	\label{eq:Breve_Sigma_K_bound}
	\|\Breve{\Sigma}_{K}\| &\leq \sigma^2 +(1+\|K\|_{\mathrm{F}}^{2})\cdot \|{\Sigma}_{K}\|.
	\end{align}
	Furthermore, since $\rho(\breve K)<1$, Lemma \ref{lemma:beta_mixing} from \citet{tu2017least} implies $\{v_t\}_{t\geq0}$ is a geometrically $\beta$-mixing stochastic process with parameter $\rho \in (\rho(\breve K),1)$.
\end{proof}

\section{Auxiliary Lemmas}

\begin{lemma}[Hansen-Wright Inequality \citep{rudelson2013hanson}]\label{Hansen-Wright}
	Let $X=\left(X_{1}, \ldots, X_{n}\right)\sim N(0,I_n) \in \mathbb{R}^{n} $ and
	$A \in \mathbb{R}^{n \times n}$  a fixed matrix. Then it holds that
	\begin{equation}
	\mathrm{P}\left\{\left|X^{\top} A X-\mathbb{E} X^{\top} A X\right|>s\right\} \leq 2 \exp \left[-C \min \left(s^2\|A\|_{\mathrm{F}}^{-2}\|, s\|A\|^{-1}\right)\right],
	\end{equation}
	where $C$ is an absolute constant.
\end{lemma}

\begin{lemma}[Lemma B.2 in \cite{yang2019global}]\label{lemma:pe_mat}
	Suppose $\rho(A-B K)<1$. Let $N(0,\Breve{\Sigma}_{K})$ be the stationary distribution 
	of $v$ specified in Lemma~\ref{lemma:chv}. Then $\Theta_{K}=\mathbb{E}_{v}\left\{\varphi(v)\left[\varphi(v)-\varphi\left(v^{\prime}\right)\right]^{\top}\right\}$ 
	is invertible and can be written as 
	\begin{equation}
	\Theta_{K}=\left(\Breve{\Sigma}_{K} \otimes_{s} \Breve{\Sigma}_{K}\right)-\left(\Breve{\Sigma}_{K} \Breve K^{\top}\right) \otimes_{s}\left(\Breve{\Sigma}_{K} \Breve K^{\top}\right)=\left(\Breve{\Sigma}_{K} \otimes_{s} \Breve{\Sigma}_{K}\right)\left(I-\Breve K^{\top} \otimes_{s} \Breve K^{\top}\right),
	\end{equation}
	where we use $A \otimes_{s} B$ to denote the symmetric Kronecker product of matrices $A$ and $B$. Furthermore, we have $\|\Theta_{K}\|\leq4\left(1+\|K\|_{\mathrm{F}}^{2}\right)^{2} \cdot\left\|\Sigma_{K}\right\|^{2}$, 
	and the matrix $\Omega_K$, defined in \eqref{eq:Ometa_gamma}, 
	has the minimum singular value lower bounded by a constant 
	$\tau_{K}^{*}>0$ that only depends on $\rho(A-B K)$, $\sigma$, and $\sigma_{\min }(\Phi)$.
\end{lemma}

\begin{lemma}[Lemma 5.4 in \cite{yang2019global}]\label{lemma:saddle_gap}
	Consider the minimax stochastic optimization problem with convex-concave objective function $H(x,y)$ defined by
	\begin{equation}
	\min _{x \in \mathcal{X}_x} \max _{y \in \mathcal{X}_y} H(x, y)=\mathbb{E}_{\zeta \sim \pi_{\zeta}}[\Psi(x, y ; \zeta)].
	\end{equation}
	Assume $\mathcal{X}_x$ and $\mathcal{X}_y$ are convex, $\|x-x^\prime\|_2\leq D$ for all $x, x^\prime \in\mathcal{X}_x$ and $\|y-y^\prime\|_2\leq D$ for all $y, y^\prime \in \mathcal{X}_y$, when $D>0$ is a constant. 
	Moreover, assume the stationary distribution $\pi_{\zeta}$ of $\zeta$ corresponds to a Markov chain that 
	has a mixing coefficients satisfying $\beta(k) \leq C_{\zeta} \cdot \rho^{k}$ for some constant $C_{\zeta}$, where 
	$\beta(k)$ is the $k$-th mixing coefficient. 
	In addition, we assume for all $\zeta\sim\pi_\zeta$, 
	the objective function $\Psi(x, y ; \zeta)$ is $L_1$-Lipschitz in both $x$ and $y$ almost surely, 
	$\nabla_{x} \Psi(x, y ;\zeta)$ is $L_2$-Lipschitz in y for all $x \in \mathcal{X}_x$, 
	and $\nabla_{y} \Psi(x, y ;\zeta)$ is $L_2$-Lipschitz in x for all $y \in \mathcal{X}_y$ 
	for some constant $L_1$ and $L_2$. Without loss of generality, we consider the case where the constants 
	$D$, $L_1$, and $L_2$ are all greater than $1$.

	Let $\mathcal{P}_{\mathcal{X}_{x}}$ and $\mathcal{P}_{\mathcal{X}_{y}}$ be the projection operators. 
	Consider the Gradient-based TD (GTD) algorithm with iterates
	\begin{align}
	x_{t}&=\mathcal{P}_{\mathcal{X}_{x}}\left[x_{t-1}-\alpha_{t} \nabla_{x} \Psi\left(x_{t-1}, y_{t-1} ; \zeta_{t-1}\right)\right],\\
	y_{t}&=\mathcal{P}_{\mathcal{X}_{y}}\left[y_{t-1}+\alpha_{t} \cdot \nabla_{y} \Psi\left(x_{t-1}, y_{t-1} ; \zeta_{t-1}\right)\right],
	\end{align}
	where step sizes $\alpha_{t} = \alpha/\sqrt{t}$, for $t \in [T]$, that returns
	\begin{align}
	\widehat{x}=\frac{\sum_{t=1}^T \alpha_{t} \cdot x_{t}}{\sum_{t=1}^T \alpha_{t}}, \quad \widehat{y}=\frac{\sum_{t=1}^T \alpha_{t} \cdot y_{t}}{\sum_{t=1}^T \alpha_{t}}
	\end{align}
	as the final output. Then there exists an absolute constant 
	$M > 0$ such that for any $ \delta \in (0,1)$, with probability at least $1-\delta$, 
	the primal-dual gap can be bounded as  
	\begin{align}
	\max _{y \in \mathcal{X}_y} \tilde G(\widehat{x}, y)-\min_{x \in \mathcal{X}_x} \tilde G (x, \widehat{y}) \leq \frac{M \cdot\left(D^{2}+L_{1}^{2}+L_{1} L_{2} D\right)}{\log (1 / \rho)} \cdot \frac{\log ^{2} T+\log (1 / \delta)}{\sqrt{T}}+\frac{M \cdot C_{\zeta} L_{1} D}{T}.
	\end{align}
\end{lemma}

\begin{lemma}[Proposition 3.1 in \cite{tu2017least}]\label{lemma:beta_mixing}
	Let $X_{t+1} = AX_t+w_{t}$ be a linear dynamic system, where noise $w_{t}$ has a 
	Gaussian distribution, and $A\in \mathbb{R}^{n\times n}$ has spectral norm $\rho(A)<1$. 
	Denote the stationary distribution of $\{X_{t}\}_{t\geq0}$ by $N(0, \Sigma_A)$. 
	For any integer $k\geq0$, the $k$-th $\beta$-mixing coefficient is defined as
	\begin{equation}
	\beta(k)=\sup _{t \geq 0} \mathbb{E}_{x \sim \mathcal{D}_{t}}\left[\left\|\mathbb{P}_{X_{k}}\left(\cdot | X_{0}=x\right)-\mathbb{P}_{N\left(0, \Sigma_A\right)}(\cdot)\right\|_{\mathrm{TV}}\right],
	\end{equation}
	where the expectation is taken with respect to the marginal distribution $\mathcal{D}_{t}$ of $X_t$.
	Then it holds that for any $k\geq$ and $\rho \in (\rho(A),1)$, 
	\begin{equation}
	\beta(k) \leq C_{\rho, A} \cdot\left[\tr\left(\Sigma_{A}\right)+n \cdot(1-\rho)^{-2}\right]^{1 / 2} \cdot \rho^{k},
	\end{equation}
	where $C_{\rho, A}$ is a constant that depends on $\rho$ and $A$ only. 
	Therefore, $\{X_{t}\}_{t\geq0}$ is geometrically $\beta$-mixing.
\end{lemma}

\begin{lemma}[\cite{nagar1959bias}]\label{lemma:quadform}
	Let $x\sim N(0,I_n)$, and 
	let $M$, $N$ be two symmetric matrices in $\mathbb{R}^{n\times n }$. It holds that,
	\begin{equation}
	\mathbb{E}\left[x^{\top} M x \cdot x^{\top} N x\right]=2 \tr\left(M N\right)+\tr\left(M\right) \cdot \tr\left(N\right).
	\end{equation}
\end{lemma}

\newpage
\section{ Notation Table}
\label{sec:notation_table}

\begin{table}[h]
  \captionsetup{justification=centering}

\caption{General Notation}
\begin{tabularx}{\textwidth}{rl}
\hline
  Notation &   Meaning \hfill \\
\hline
\hline
 $\mathcal{N}$ &   Set of agents of the multiagent system.\\
  $\mathcal{X}, \mathcal{U}, \mathcal{W}$ &  State, action and noise spaces. $\mathcal{X} = \mathbb{R}^{d}$, $\mathcal{U} = \mathbb{R}^{k}$, $\mathcal{W} = \mathbb{R}^{d}$.\\
  $A,B,Q,R$ &   System matrices of the multiagent system. \\
   $\mathbf{x}_{t+1}, \mathbf{u}_{t+1}, \mathbf{w}_{t+1}$ &   Aggregated state, action and noise vectors of the   multiagent system. \\
 &  at time $t$.\\
   $\mathcal{N}^i$ &   Set of agents of the $i$-th subpopulation.\\
   $x_{t}^{i}, u_{t}^{i}, w_{t}^{i}$&  State, action and noise vectors of the  $i$-th agent.\\
   
    $\mathcal{X}^{i}, \mathcal{U}^{i}, \mathcal{W}^{i}$&   State, action and noise spaces of the $i$-th subpopulation. $\mathcal{X}^{i}=\mathbb{R}^{d_{i}}$, \\
  &$\mathcal{U}^{i}=\mathbb{R}^{k_{i}}$, and $\mathcal{W}^{i}=\mathbb{R}^{d_{i}}$.\\
  
$\bar{x}_{t}^{l}, \bar{u}_{t}^{l}, \bar{w}_{t}^{l}$&Mean-field state, action and noise of  the $i$-th subpolulation. \\ 
&$\bar{x}_{t}^{l}:=\frac{1}{\left|\mathcal{N}^{l}\right|} \sum_{i \in \mathcal{N}^{l}} x_{t}^{i}, \quad \bar{u}_{t}^{l}:=\frac{1}{\left|\mathcal{N}^{l}\right|} \sum_{i \in \mathcal{N}^{l}} u_{t}^{i}, \quad  \bar{w}_{t}^{l}:=\frac{1}{\left|\mathcal{N}^{l}\right|} \sum_{i \in \mathcal{N}^{l}} w_{t}^{i}$, \\
$\tilde{x}_{t}^{i}, \tilde{u}_{t}^{i}, \tilde{w}_{t}^{i}$ &  State, action and noise of agent $i$  in auxiliary system $\mathcal{S}_{l}$.\\
& $\tilde{x}_{t}^{i} = {x}_{t}^{i} - \bar{x}_{t}^{l},\quad \tilde{u}_{t}^{i} = {u}_{t}^{i} - \bar{u}_{t}^{l}, \quad \tilde{w}_{t}^{i} = {w}_{t}^{i} - \bar{w}_{t}^{l}$. \\
 $ \bar{\mathbf{x}}_{t}, \bar{\mathbf{u}}_{t},\bar{\mathbf{u}}_{t}$& Aggregated mean-field state, action and noise vector.  \\
 &$\bar{\mathbf{x}}_{t}:=\operatorname{vec}\left(\bar{x}_{t}^{1}, \ldots, \bar{x}_{t}^{L}\right),\  \bar{\mathbf{u}}_{t}:=\operatorname{vec}\left(\bar{u}_{t}^{1}, \ldots, \bar{u}_{t}^{L}\right), \  \bar{\mathbf{w}}_{t}:=\operatorname{vec}\left(\bar{w}_{t}^{1}, \ldots, \bar{w}_{t}^{L}\right)$.\\ 
 $\tilde{\mathbf{x}}^l_t, \tilde{\mathbf{u}}^l_t, \tilde{\mathbf{w}}^l_t $ & Tuples $\tilde{\mathbf{x}}^l_t = (\tilde{x}_{t}^{i})_{i \in \mathcal{N}^l}, \tilde{\mathbf{u}}^l_t = (\tilde{u}_{t}^{i})_{i \in
		\mathcal{N}^l}$ and $\tilde{\mathbf{w}}^l_t = (\tilde{w}_{t}^{i})_{i \in
		\mathcal{N}^l}$. \\
		
$M^{i,j}$& The $(i,j)$-block of a matrix $M$.\\
$a^l, b^l$ & Diagonal blocks. $a^l = A^{i,i}=A^{j,j}, \ b^l = B^{i,i}=B^{j,j},\  \forall i,j \in \mathcal{N}^l$.\\
$\breve{a}^{l,k}, \breve{b}^{l,k}$ & Off-diagonal blocks.\\
& $\bar{a}^{l,k} = A^{i,n}=A^{j,m}, \ \bar{b}^{l,k} = B^{i,n}=B^{j,m},\  \forall i,j \in \mathcal{N}^l, n,m \in \mathcal{N}^k$.\\ 
$q^l, r^l$ & Diagonal blocks. $q^l = Q^{i,i}=Q^{j,j}, \ r^l = R^{i,i}=R^{j,j},\  \forall i,j \in \mathcal{N}^l$.\\
$\breve{q}^{l,k}, \breve{r}^{l,k}$ & Off-diagonal blocks. \\
&$\breve{q}^{l,k} = Q^{i,n}=Q^{j,m}, \ \breve{r}^{l,k} = R^{i,n}=R^{j,m},\  \forall i,j \in \mathcal{N}^l, n,m \in \mathcal{N}^k$.\\ 

$A_{l}, B_l$&Dynamics of the $l$-th auxiliary system. $A_{l}=a^{l}-\breve{a}^{l, l}$, $B_{l}=b_{t}^{l}-\breve{b}^{l, l}$.\\
$Q_{l}, R_{l}$& Cost matrices of the mean-field agent.
 $Q_{l}:=q^{l}-\breve{q}^{l, l}, \quad R_{l}:=r^{l}-\breve{r}^{l, l}$.\\
 
$\breve{A}_{l}$, $\breve{B}_{l}$ & $\breve{A}_{l}:=\operatorname{cols}\left(\left|\mathcal{N}^{1}\right| \breve{a}^{l, 1}, \ldots,\left|\mathcal{N}^{L}\right| \breve{a}^{l, L}\right)$, $\breve{B}_{l}:=\operatorname{cols}\left(\left|\mathcal{N}^{1}\right| \breve{b}^{l, 1}, \ldots,\left|\mathcal{N}^{L}\right| \breve{b}^{l, L}\right)$.\\
$\bar{A}, \bar{B}$& Dynamics of the mean-field agent.
\\
&$\bar{A}  :=\operatorname{diag}\left(A_{1}, \ldots, A_{L}\right)+\operatorname{rows}\left(\breve{A}_{1}, \ldots, \breve{A}_{L}\right)$, \\
&$\bar{B}:=\operatorname{diag}\left(B_{1}, \ldots, B_{L}\right)+\operatorname{rows}\left(\breve{B}_{1}, \ldots, \breve{B}_{L}\right)$.\\
$\breve{Q}^{l, k}, \breve{R}^{l, k}$&$\breve{Q}^{l, k}:=\left|\mathcal{N}^{l}\right|\left|\mathcal{N}^{k}\right| \breve{q}^{l, k} \quad \breve{R}^{l, k}:=\left|\mathcal{N}^{l}\right|\left|\mathcal{N}^{k}\right| \breve{r}^{l, k}$.\\
$\breve{Q}, \breve{R}$&Matrices whose $(i,j)$-th block are $\breve{Q}^{l, k}, \breve{R}^{l, k}$ respectively.\\

$\bar{Q}, \bar{R}$& Cost matrices of the mean-field agent.\\
&$\bar{Q}:=\breve{Q} + \operatorname{diag}\left(\left|\mathcal{N}^{1}\right| \cdot Q_{1}, \ldots,\left|\mathcal{N}^{L}\right| \cdot Q_{L}\right)$, \\
&$\bar{R}:=\breve{R} +\operatorname{diag}\left(\left|\mathcal{N}^{1}\right| \cdot R_{1}, \ldots,\left|\mathcal{N}^{L}\right| \cdot R_{L}\right)$.\\
\hline

\end{tabularx}

\end{table}

\newpage

\bibliography{rl_ref,paper} 

\end{document}